\documentclass[journal,onecolumn]{IEEEtran}

\usepackage{mathptmx}
\usepackage{amsmath}
\usepackage{amsthm,amsmath,amssymb}
\usepackage{bm} 
\usepackage{amsfonts}
\usepackage{subeqnarray}
\usepackage{cases,cite}
\usepackage[utf8]{inputenc}
\usepackage{graphicx}
\usepackage{amssymb}
\usepackage{float}
\usepackage{graphicx} 
\usepackage{subfigure}
\usepackage{algpseudocode}
\usepackage{url}
\usepackage{multirow}
\usepackage{microtype}
\usepackage{booktabs}
\usepackage{multirow}
\usepackage{color}
\usepackage[table,xcdraw]{xcolor}

\newtheorem{proposition}{Proposition}
\newtheorem{definition}{Definition}

\newtheorem{lemma}{Lemma}

\newtheorem{assumption}{Assumption}

\IEEEoverridecommandlockouts

\begin{document}

\title{A Game-Theoretic Approach to Design Secure and Resilient Distributed Support Vector Machines}

\author{Rui~Zhang,~\IEEEmembership{Student Member}
        and~Quanyan~Zhu,~\IEEEmembership{Member}

\thanks{R. Zhang and Q. Zhu are with Department of Electrical and Computer Engineering, New York University, Brooklyn, NY, 11201
E-mail:\{rz885,qz494\}@nyu.edu. 

This research is partially supported by a DHS grant through Critical Infrastructure Resilience Institute (CIRI), grants CNS-1544782 and SES-1541164 from National Science of Foundation (NSF), and grant DE-NE0008571 from the Department of Energy (DOE).

A preliminary version of this work has been presented at the 18th International Conference on Information fusion, Washington, D.C., 2015 \cite{zhang2015secure}.}}

\maketitle

\begin{abstract}
Distributed Support Vector Machines (DSVM) have been developed to solve large-scale classification problems in networked systems with a large number of sensors and control units. However, the systems become more vulnerable as detection and defense are increasingly difficult and expensive. This work aims to develop secure and resilient DSVM algorithms under adversarial environments in which an attacker can manipulate the training data to achieve his objective. We establish a game-theoretic framework to capture the conflicting interests between an adversary and a set of
distributed data processing units. The Nash equilibrium of the game allows predicting the outcome of learning algorithms in adversarial environments, and enhancing the resilience of the machine learning through dynamic distributed learning algorithms.
We prove that the convergence of the distributed algorithm is guaranteed without assumptions on the training data or network topologies. Numerical experiments are conducted to corroborate
the results. We show that network topology plays an important role in the
security of DSVM. Networks with fewer nodes and higher average degrees are more secure. Moreover, a balanced network is found to be less vulnerable to attacks.
\end{abstract}


\section{Introduction}
{
Support Vector Machines (SVMs)\cite{suykens1999least} have been widely used for classification and prediction tasks, such as spam detection\cite{sculley2007relaxed}, face recognition\cite{osuna1997training} and temperature prediction\cite{radhika2009atmospheric}. They are supervised learning algorithms that can be used for prediction or detection by training samples with known labels. However, just like many other machine learning algorithms, SVMs are also vulnerable to adversaries who can exploit the systems\cite{barreno2010security}. For example, an SVM-based spam filter will misclassify spam emails after training wrong data created intentionally by attacker\cite{nelson2009misleading,bruckner2009nash,bruckner2011stackelberg}. Moreover, an SVM-based face recognition systems may give wrong authentications to fake images created by attacker \cite{erdogmus2013spoofing}. }

Traditional SVMs are learning algorithms that require a centralized data collection, communication, and storage from multiple sensors \cite{flouri2006training}. The centralized nature of SVMs requires a significant amount of computation for large-scale problems, and makes SVMs unsuitable for online information fusion and processing.
{Despite the fact that various solutions have been introduced to address this challenge, e.g., see \cite{tsang2005core} and \cite{papadonikolakis2012novel}, they have not changed the nature of the SVM algorithm and its architecture. }

Distributed Support Vector Machines (DSVM) algorithms are decentralized SVMs in which multiple nodes or agents process data independently, and communicate training information over a network, see, for example, \cite{do2006classifying,meligy2009grid}. This architecture is attractive for solving large-scale machine learning problems since each node learns from its own data in parallel, and transfers the learning results from one node to the others to achieve the global performance as in centralized algorithms. In addition, DSVM algorithms do not require a fusion center to store all the data. Each node performs its local computation without sharing the content of the data with other nodes, which effectively reduces the cost of memory and the overhead of data communications. 

In spite of the productivity and efficiency of DSVM, the decentralized training system is more vulnerable than its centralized counterpart \cite{zhang2017cdc, zhang2017game}. The DSVM has an increased attack surface since each node in the network can be vulnerable to attacks. An attacker can not only select a few nodes to compromise their individual learning process \cite{kavitha2010security}, but also send misinformation to other nodes to affect the performance of the entire DSVM network \cite{wang2007toward}. In addition, in the case of large-scale problems, it is not always possible to protect a large number of nodes at the same time \cite{anderson2001information}. Hence there will always exist vulnerabilities so that an attacker can find the weakest links or nodes to compromise. 

As a result, it is important to study the security of DSVM under adversarial environments. In this work, we focus on a class of consensus-based DSVM algorithms \cite{forero2010consensus}, in which each node in the network updates its training result based on its own training data and the results from its neighboring nodes. Nodes  achieve the global training results once they reach consensus. One compromised node will play a significant role in affecting not only its own training result but spreading the misinformation to the entire network.

Machine learning algorithms are inherently vulnerable as they are often open-source tools or methods, and security is not the primary concern of designers. An attacker can easily acquire the information regarding the DSVM algorithms and the associated network topologies. {With this knowledge, an attacker can launch a variety of attacks, for example, manipulating the labels of the training samples \cite{frenay2014classification}, and changing the testing data \cite{moreno2012unifying}.} In this work, we consider a class of attacks in which the attacker has the ability to modify the training data. An example of this has been described in \cite{mei2015using}, where an adversary modifies training data so that the learner is misled to produce a prediction model profitable to the adversary. This type of attack represents a challenge for the learner since it is hard to detect data modifications during a training process \cite{rndic2014practical}. We further identify the attacker by his goal, knowledge, and capability. 

\begin{itemize}
\item {\it The Goal of the Attacker}: The attacker aims to destroy the training process of the DSVM learner and increase his classification errors. 

\item {\it The Knowledge of the Attacker}: 
To fully capture the damages caused by the attacker, we assume that the attacker has a complete knowledge of the learner, i.e., the attacker knows the learner's data and algorithm and the network topology. This assumption is under a worst-case scenario by Kerckhoffs's principle: the enemy knows the system \cite{shannon1949communication}.
 
\item {\it The Capability of the Attacker}: The attacker can modify the training data by deleting crafted values to damage the training process of the DSVM learner. 
\end{itemize}

One major goal of this work is to develop a quantitative framework to address this critical issue. In the adversarial environments, the goal of a learner is to minimize global classification errors in a network, while an attacker breaks the training process with the aim of maximizing that errors of classification by modifying the training data. The conflict of interests enables us to establish a nonzero-sum game framework to capture the competitions between the learner and the attacker. The Nash equilibrium of the game enables the prediction of the outcome and yields
optimal response strategies to the adversary behaviors. The game framework also provides a theoretic basis for developing dynamic learning algorithms that will enhance the security and
the resilience of DSVM.  
The major contribution of this work can be summarized as follows:
\begin{enumerate}
\item We capture the attacker's objective and constrained capabilities in a game-theoretic framework and develop a nonzero-sum game to model the strategic interactions between an attacker and a learner with a distributed set of nodes.
\item We fully characterize the Nash equilibrium by showing the strategic equivalence between the original nonzero-sum game and a zero-sum game.
\item We develop secure and resilient distributed algorithms based on alternating direction method of multipliers (ADMoM)\cite{eckstein2012augmented}. Each node communicates with its neighboring nodes and updates its decision strategically in response to adversarial environments.
\item We prove the convergence of the DSVM algorithm. The convergence is guaranteed without any assumptions on the network topology or the form of data. 
\item We demonstrate that network topology plays an important role in resilience to adversary behaviors. Networks with fewer nodes and higher average degrees are shown to be more secure. We also show that a balanced network (i.e., each node has the same number of neighbors) is less vulnerable.
\item We show that nodes with more training samples and fewer neighbors turn out to be more secure for a specified network. One way to defend against attacker's action is to add more training samples, which may increase the training time and require more memory for storage.
\end{enumerate}
\subsection{Related Works}{
A general tool to study machine learning under adversarial environment is game theory\cite{dalvi2004adversarial,kantarciouglu2011classifier,rota2016randomized}. 
In \cite{dalvi2004adversarial}, Dalvi et al. have formulated a game between a cost-sensitive Bayes classifier and cost-sensitive adversary. In \cite{kantarciouglu2011classifier}, Kantarc{\i}o{\u{g}}lu et al. have introduced Stackelberg games to model the interactions between the adversary and the learner, which shows that the game between them is possible to reach a steady state where actions of both players are stabilized. In \cite{rota2016randomized}, Rota et al. have presented a game-theoretic formulation where a learner and an attacker make randomized strategy selections. The major focus of their work is on developing centralized machine learning tools. In our work, we extend the security framework of machine learning algorithms to a distributed framework for networks. Hence, it can be seen that the performance of the distributed machine learning algorithms is also related the security of networks. }

Game theory has also been widely used in network security \cite{lye2005game,michiardi2002game,zhu2010heterogeneous,zhu2011distributed,zhu2013game,zhu2015game,manshaei2013game,zhang2017bi}. In \cite{lye2005game}, Lye et al. have analyzed the interactions of an attacker and an administrator as a two-player stochastic game at a network. In \cite{michiardi2002game}, Michiardi et al. have presented a game-theoretic model in ad hoc networks to capture the interactions between normal nodes and misbehaving nodes.  However, when solving distributed machine learning problems, the features and properties of data processing in each node can cause unanticipated consequences in a network. 

In our previous work \cite{zhang2015secure}
, we have established a preliminary  framework to model the interactions between a consensus-based DSVM learner and an attacker. In this paper, we develop fully distributed algorithms and investigate their convergence, security and resilience properties. Moreover, new sets of experiments are performed to show the influence of network topologies and the number of samples at each node on the resilience of the network.

\subsection{Organization of the Paper}
The rest of this paper is organized as follows. Section \ref{sec:Pre} outlines the design of distributed support vector machines. In Section \ref{sec:DSVMA}, we establish game-theoretic models for the learner and the attacker. Section \ref{sec:ADMoMDSVM} deals with the distributed and dynamic algorithms for the learner and the attacker. Section \ref{sec:Convergence} presents the convergence proof of the algorithm. Section \ref{sec:Num} and Section \ref{Sec:Con} present numerical results and concluding remarks, respectively. Appendices A, B, and C provide the proof of the Propositions 1, 2 and Lemma 1, respectively. 
\subsection{Summary of Notations}
Notations in this paper are summarized as follows. Boldface letters are used for matrices (column vectors); $(\cdot)^T$ denotes matrix and vector transposition;
$(\cdot)^{(t)}$ denotes values at step $t$; $[\cdot]_{vu}$ denotes the $vu$-th entry of a matrix; $diag(\mathbf{X})$ is the diagonal matrix with $\mathbf{X}$ on its main diagonal; $\parallel \cdot \parallel$ is the norm of the matrix or vector; $\mathcal{V}$ denotes the set of nodes in a network; $\mathcal{B}_v$ denotes the set of neighboring nodes of node $v$; $\mathcal{U}$ denotes the action set which is used by the attacker.
\section{PRELIMINARIES}
\label{sec:Pre}
In this section, we present a two-player machine learning game in a distributed network involving a learner and an attacker to capture the strategic interactions between them. The network is modeled by an undirected graph $\mathcal{G(V,E)}$ with $\mathcal{V}:=\lbrace 1,...,V \rbrace$ representing the set of nodes, and $\mathcal{E}$ representing the set of links between nodes. Node $v\in \mathcal{V}$ communicates only with his neighboring nodes $\mathcal{B}_v\subseteq\mathcal{V}$. Note that without loss of generality, graph $\mathcal{G}$ is assumed to be connected; in other words, any two nodes in graph $\mathcal{G}$ are connected by a path. However, nodes in $\mathcal{G}$ do not have to be fully connected, which means that nodes are not required to directly connect to all the other nodes in the network. The network can contain cycles. At every node $v\in \mathcal{V}$, a labelled training set $\mathcal{D}_v:= \lbrace(\mathbf{x}_{vn}, y_{vn}):n=1,...,N_v\rbrace$ of size $N_v$ is available, where $\mathbf{x}_{vn} \in \mathbb{R}^p$ represents a
$p$-dimensional {data}, and they are divided into two groups with labels $y_{vn} \in \{+1,-1\}$. Examples of a network of distributed nodes are illustrated in Fig. \ref{fig:networkexample}(a). 
\begin{figure}[]
\centering
\subfigure[Network example.]{
\includegraphics[width=0.3\textwidth]{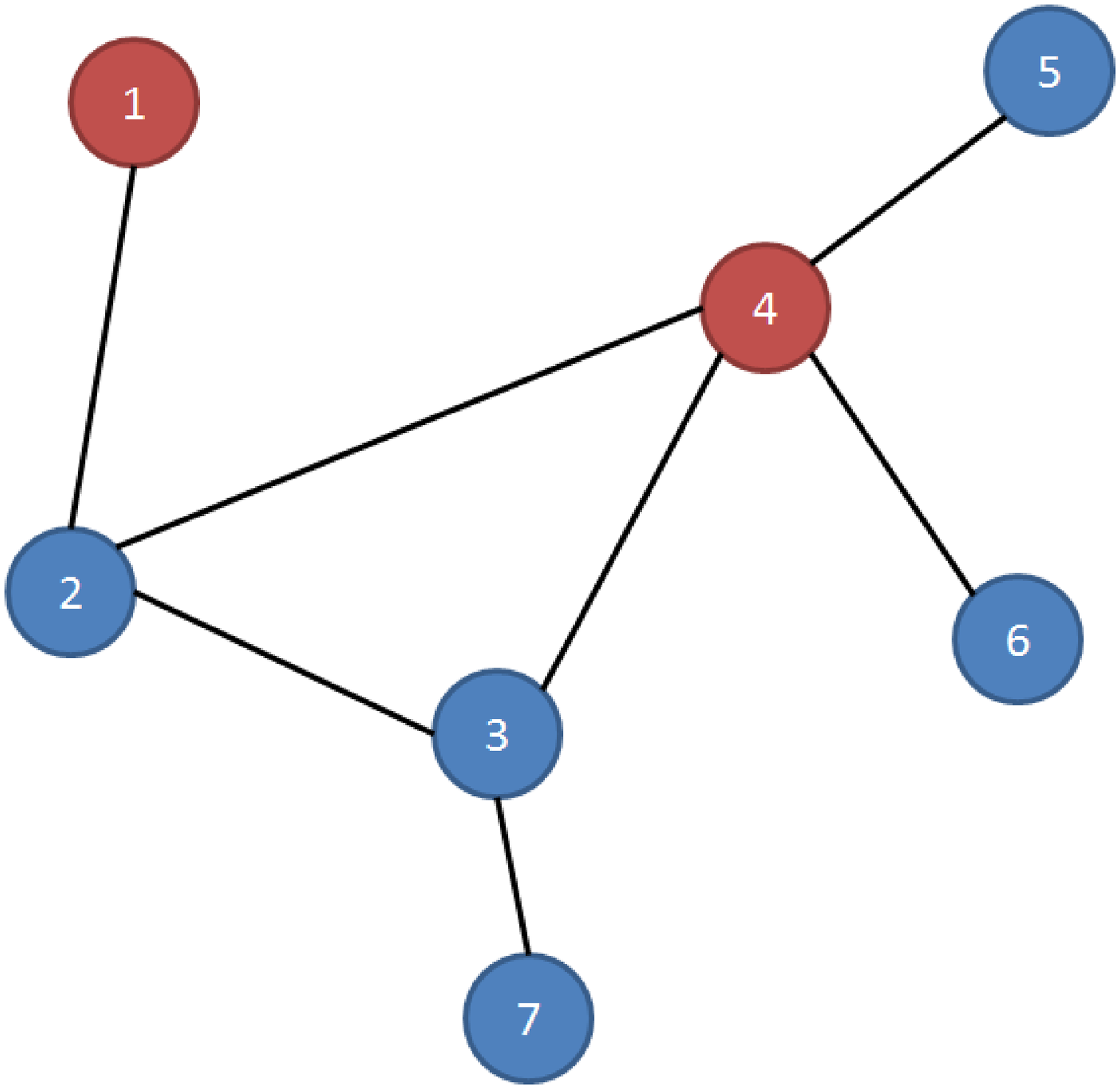}}
\subfigure[SVM at compromised node $1$.]{
\includegraphics[width=0.3\textwidth]{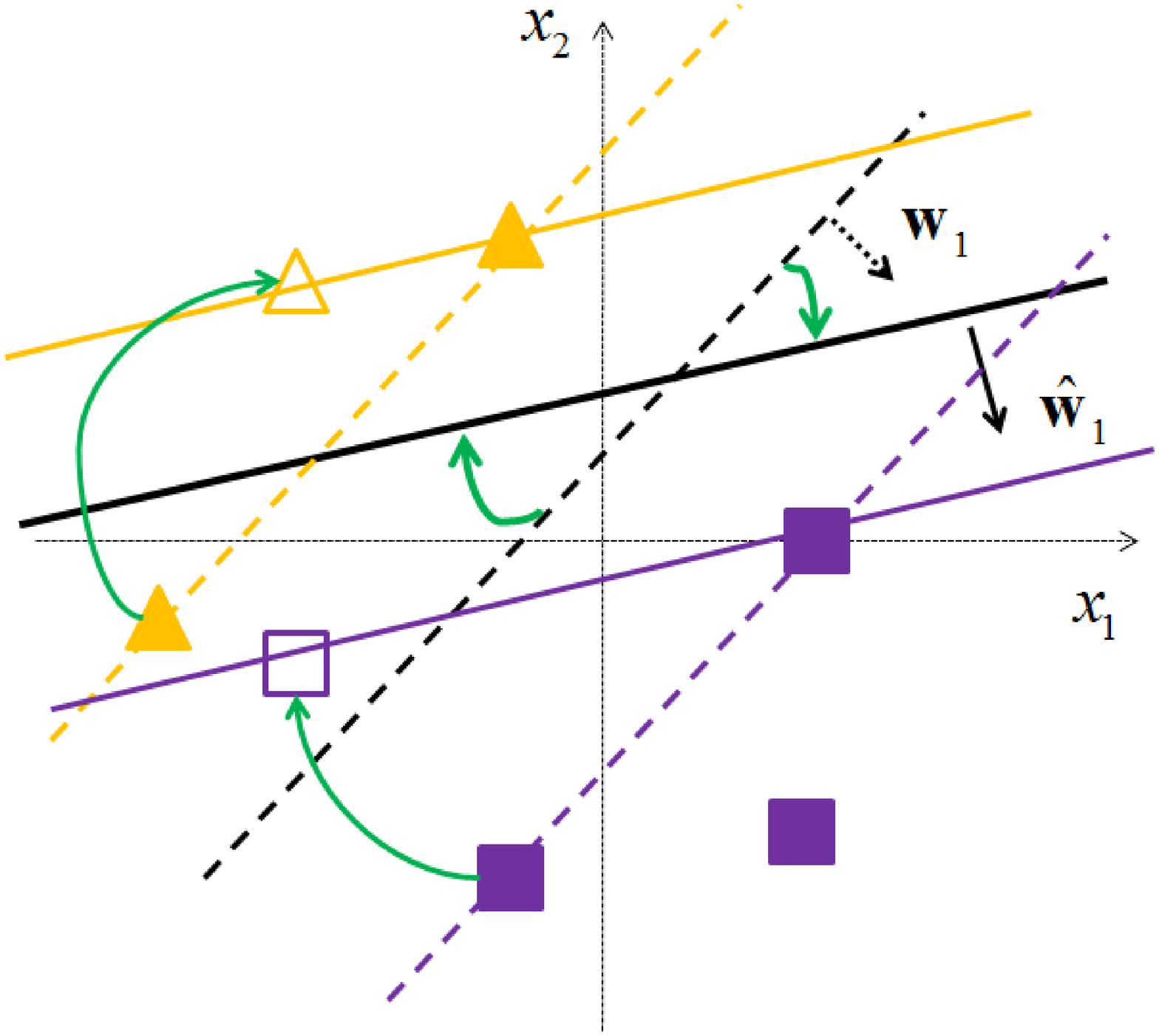}}
\vspace{-2.5mm}
\caption{Network example: There are $7$ nodes in this network as shown in Fig. (a). Each node contains a labelled training set $\mathcal{D}_v:= \lbrace(\mathbf{x}_{vn}, y_{vn}):n=1,...,N_v\rbrace$. Node $4$ can communicate with its 4 neighbors: node $2$, $3$, $5$ and $6$. An attacker can take over node $1$ and $4$. The compromised nodes are marked in red. {In each node, the learner aims to find the best linear discriminant line, for example, the black dotted line shown in (b). In compromised nodes, an attacker modifies the training data which leads to a wrong discriminant line of the learner, for example, the black solid line shown in (b)}.}
\label{fig:networkexample}
\end{figure}

The goal of the learner is to design DSVM algorithms for each node in the network based on its local training data $\mathcal{D}_v$, so that each node has the ability to give new input $\mathbf{x}$ a label of $+1$ or $-1$ without communicating $\mathcal{D}_v$ to other nodes $v^\prime \neq v$. To achieve this, the learner aims to find local maximum-margin linear discriminant functions $ g_v(\mathbf{x}) = \mathbf{x}^T \mathbf{w}_v^{*} + b_v^{*}$ at every node $v\in\mathcal{V}$ with the consensus constraints $\{{{\bf{w}}_v^*} = {{\bf{w}}_u^*},{b_v^*} = {b_u^*}\}_{v\in\mathcal{V},u\in\mathcal{B}_v}$ forcing all the local variables $\{ \mathbf{w}_v^*, b_v^*\}$ to agree across neighboring nodes. Variables $\mathbf{w}_v^*$ and $b_v^*$ of the local discriminant functions $g_v(\mathbf{x})$ can be obtained by solving the following convex optimization problem \cite{forero2010consensus}:
\begin{equation}
\label{eq:DSVM}
\begin{array}{l}
\begin{array}{*{20}{l}}
{\mathop {\min }\limits_{\left\{ {{{\bf{w}}_v},{b_v}} \right\}} \frac{1}{2}\sum\limits_{v \in \mathcal{V}} {{{\left\| {{{\bf{w}}_v}} \right\|^2_2}}} }\\
{\begin{array}{*{20}{c}}
{}
\end{array} + V{C_l}\sum\limits_{v \in \mathcal{V}} {\sum\limits_{n = 1}^{{N_v}} {{{\left[ {1 - {{\rm{y}}_{vn}}({\bf{w}}_v^T{{\bf{x}}_{vn}} + {b_v})} \right]}_ + }} } }
\end{array}\\
\begin{array}{*{20}{c}}
{{\rm{s}}.{\rm{t}}.}&{{{\bf{w}}_v} = {{\bf{w}}_u},}&{{b_v} = {b_u},}&{\forall v \in {\cal V},u \in {B_v}.}
\end{array}
\end{array} 
\end{equation}
{In the above problem, the term ${{{\left[ {1 - {{\rm{y}}_{vn}}({\bf{w}}_v^T{{\bf{x}}_{vn}} + {b_v})} \right]}_ + }}:={\max \{1 - {{\rm{y}}_{vn}}({\bf{w}}_v^T{{\bf{x}}_{vn}} + {b_v}),0\}}$ is the hinge loss function.} It can also be written as slack variable $\xi_{vn}$ with the constraints ${{{\rm{y}}_{vn}}({\bf{w}}_v^T{{\bf{x}}_{vn}} + {b_v}) \ge 1 - {\xi _{vn}}}$ and ${{\xi _{vn}} \ge 0}$, where $\xi_{vn}$ account for non-linearly separable training sets. $C_l$ is a tunable positive scalar for the learner.
\section{Distributed Support Vector Machines with Adversary} 
\label{sec:DSVMA}
Optimization Problem (\ref{eq:DSVM}) is formed by the DSVM learner who seeks to find the maximum-margin linear discriminant function. 
We assume that an attacker has a complete knowledge of the learner's Problem (\ref{eq:DSVM}), and he can modify the value $\mathbf{x}_{vn}$ of the node $v$ into ${\widehat {\bf{x}}_{vn}} = \mathbf{x}_{vn}-\delta_{vn}$, where $\delta_{vn} \in \mathcal{U}_v$, and $\mathcal{U}_v$ is the attacker's action set at node $v$. We use $\mathcal{V}_a=\{1,...,V_a\} $ and $\mathcal{V}_l = \{1,...,V_l\}$ to represent nodes with and without the attacker, respectively. Note that, $V = V_a + V_l$ and $\mathcal{V}=\mathcal{V}_l\cup \mathcal{V}_a$. A node in the network is either under attack or not under attack. The behavior of the learner can be captured by the following optimization problem:
\begin{equation}
\label{eq:aDSVM1}
\begin{array}{l}
\begin{array}{*{20}{l}}
{\mathop {\min }\limits_{\left\{ {{{\bf{w}}_v},{b_v}} \right\}} \frac{1}{2}\sum\limits_{v \in \mathcal{V}} {{{\left\| {{{\bf{w}}_v}} \right\|_2^2}}} }\\
{\begin{array}{*{20}{l}}
{\begin{array}{*{20}{c}}
{}&{}
\end{array} + {V_l}{C_l}\sum\limits_{v \in \mathcal{V}_l} {\sum\limits_{n = 1}^{{N_v}} {{{\left[ {1 - {{\rm{y}}_{vn}}({\bf{w}}_v^T{{\bf{x}}_{vn}} + {b_v})} \right]}_ + }} } }\\
{\begin{array}{*{20}{c}}
{}&{}
\end{array} + {V_a}{C_l}\sum\limits_{v \in \mathcal{V}_a} {\sum\limits_{n = 1}^{{N_v}} {{{\left[ {1 - {{\rm{y}}_{vn}}({\bf{w}}_v^T({{\bf{x}}_{vn}} - {\delta _{vn}}) + {b_v})} \right]}_ + }} } }
\end{array}}
\end{array}\\
\begin{array}{*{20}{c}}
{{\rm{s}}.{\rm{t}}.}&{{{\bf{w}}_v} = {{\bf{w}}_u},}&{{b_v} = {b_u},}&{\forall v \in {\cal V},u \in {{\cal B}_v}.}
\end{array}
\end{array}
\end{equation}
For the learner, the learning process is to find the discriminant function which separates the training data into two classes with less error, and then use the discriminant function to classify testing data. Since the attacker has the ability to change the value of the original data $\mathbf{x}_{vn}\in \mathcal{X}$ into ${\widehat {\bf{x}}_{vn}} \in \widehat{\mathcal{X}}$, the learner will find the discriminant function that separates the data in $\widehat{\mathcal{X}}$ more accurate, rather than the data in $\mathcal{X}$. As a result, when using the discriminant function to classify the testing data $\mathbf{x}\in\mathcal{X}$, it will be prone to be misclassified. 

By minimizing the objective function in Problem (\ref{eq:aDSVM1}), the learner can obtain the optimal variables $\{\mathbf{w}_v^*,b_v^*\}$, which can be used to build up the discriminant function to classify the testing data. The attacker, on the other hand, aims to find an optimal way to modify the data using variables $\{\delta_{vn}\}$ to maximize the classification error of the learner. The behavior of the attacker can thus be captured as follows:
\begin{equation}
\label{eq:aDSVM}
\begin{array}{l}
\begin{array}{*{20}{l}}
{\mathop {\max }\limits_{\left\{ {{\delta _{vn}}} \right\}} \frac{1}{2}\sum\limits_{v \in \mathcal{V}} {{{\left\| {{{\bf{w}}_v}} \right\|_2^2}}} }\\
{\begin{array}{*{20}{c}}
{}&{}
\end{array} + {V_l}{C_l}\sum\limits_{v \in \mathcal{V}_l} {\sum\limits_{n = 1}^{{N_v}} {{{\left[ {1 - {{\rm{y}}_{vn}}({\bf{w}}_v^T{{\bf{x}}_{vn}} + {b_v})} \right]}_ + }} } }\\
{\begin{array}{*{20}{c}}
{}&{}
\end{array} + {V_a}{C_l}\sum\limits_{v \in \mathcal{V}_a} {\sum\limits_{n = 1}^{{N_v}} {{{\left[ {1 - {{\rm{y}}_{vn}}({\bf{w}}_v^T({{\bf{x}}_{vn}} - {\delta _{vn}}) + {b_v})} \right]}_ + }} } }\\
{\begin{array}{*{20}{c}}
{}&{}
\end{array} - {C_a}\sum\limits_{v \in \mathcal{V}_a} {\sum\limits_{n = 1}^{{N_v}} {{{\left\| {{\delta _{vn}}} \right\|}_0}} } }
\end{array}\\
{\rm{s}}{\rm{.t}}{\rm{.   }}\begin{array}{*{20}{c}}
{({\delta _{v1}},...,{\delta _{v{N_v}}}) \in {{\cal U}_v}},&{\forall v \in {{\cal V}_a}}.
\end{array}
\end{array}
\end{equation}
In above problem, the term $ {C_a}\sum_{v \in \mathcal{V}_a} {\sum_{n = 1}^{{N_v}} {{{\left\| {{\delta _{vn}}} \right\|}_0}} } $ represents the cost function for the attacker. {$l_0$ norm is defined as ${\left\| \mathbf{x} \right\|_0} := \# \{i : {x_i} \neq 0\}$, i.e., the total number of nonzero elements in a vector}. Here, we use the $l_0$ norm to denote the number of elements which are changed by the attacker. The objective function with $l_0$-norm captures the fact that the attacker aims to make the largest impact on the learner by changing the least number of elements. $\mathcal{U}_v$ denotes the action set for the attacker. We use the following form of $\mathcal{U}_v$: 
\[{\cal U}_v = \left\{ {\left( {{\delta _{v1}},...,{\delta _{vN_v}}} \right)\left| {\sum\limits_{n = 1}^{N_v} {\left\| {{\delta _{vn}}} \right\|_2^2}  \le {C_{v,\delta} }} \right.} \right\},\] 
which is related to the atomic action set \[\mathcal{U}_{v0} = \left\{ {\delta_v \left| {\left\| \delta_v  \right\|_2^2 \le {C_{v,\delta} }} \right.} \right\}.\]
$C_{v,\delta}$ indicates the bound of the sum of the norm of all the changes at node $v$. A higher $C_{v,\delta}$ indicates that the attacker has a large degree of freedom in changing the value $\mathbf{x}_{vn}$. Thus training these data will lead to a higher risk for the learner. Notice that $C_{v,\delta}$ can vary at different nodes, and we use $C_{\delta}$ to represent the situation when $C_{v,\delta}$ are equal at every node. $\delta_v\in\mathbb{R}^p$ from the atomic action set has the same form with $\delta_{vn}$, but $\delta_v$ and $( \delta _{v1},...,\delta _{vN_v} )$ are bounded by same $C_{v,\delta}$. Furthermore, the atomic action set $\mathcal{U}_{v0}$ has the following properties.

$\begin{array}{l}
\begin{array}{*{20}{c}}
{({\rm P1  })}&{\bf{0}}
\end{array} \in {\mathcal{U}_{v0}};\\
\begin{array}{*{20}{c}}
{({\rm P2  })}&{{\rm{For \ any \ }}{{\bf{w}}_0} \in {\mathbb{R}^p}:}
\end{array}\\
\begin{array}{*{20}{c}}
{}&{}
\end{array}\mathop {\max }\limits_{\delta_v  \in {\mathcal{U}_{v0}}} \left[ {{\bf{w}}_0^T\delta_v } \right] = \mathop {\max }\limits_{\delta '_v \in {\mathcal{U}_{v0}}} \left[ { - {\bf{w}}_0^T\delta'_v} \right] <  + \infty .
\end{array}$

The first property (P1) states that the attacker can choose not to change the value of $\mathbf{x}_{vn}$. Property (P2) states that the atomic action set is bounded and symmetric. Here, ``bounded'' means that the attacker has the limit on the capability of changing $\mathbf{x}_{vn}$. It is reasonable since changing the value significantly will result in the evident detection of the learner. 

Problem (\ref{eq:aDSVM1}) and Problem (\ref{eq:aDSVM}) can constitute a two-person nonzero-sum game between an attacker and a learner. The solution to the game problem is often described by Nash equilibrium, which yields the equilibrium strategies for both players, and predicts the outcome of machine learning in the adversarial environment. By comparing Problem (\ref{eq:aDSVM1}) with Problem (\ref{eq:aDSVM}), we notice that the first three terms of the objective function in Problem (\ref{eq:aDSVM}) are the same as the objective function in Problem (\ref{eq:aDSVM1}). The last term of the objective function in Problem (\ref{eq:aDSVM}) is not related to the decision of the learner when he solves Problem (\ref{eq:aDSVM1}), and thus it can be treated as a constant for the learner. Moreover, both the constraints in Problem (\ref{eq:aDSVM1}) and (\ref{eq:aDSVM}) are uncoupled. As a result, the nonzero-sum game can be reformulated into a strategically equivalent zero-sum game, which takes the minimax or max-min form as follows: 
\begin{equation}
\label{eq:alminmax}
\begin{array}{l}
\begin{array}{*{20}{l}}
{\mathop {\min }\limits_{\left\{ {{{\bf{w}}_v},{b_v}} \right\}} \mathop {\max }\limits_{\left\{ {{\delta _{vn}}} \right\}} K\left( {\left\{ {{{\bf{w}}_v},{b_v}} \right\},\left\{ {{\delta _{vn}}} \right\}} \right) \buildrel \Delta \over = \frac{1}{2}\sum\limits_{v \in \mathcal{V}} {{{\left\| {{{\bf{w}}_v}} \right\|_2^2}}} }\\
{\begin{array}{*{20}{c}}
{}&{}
\end{array} + {V_l}{C_l}\sum\limits_{v \in \mathcal{V}_l} {\sum\limits_{n = 1}^{{N_v}} {{{\left[ {1 - {{\rm{y}}_{vn}}({\bf{w}}_v^T{{\bf{x}}_{vn}} + {b_v})} \right]}_ + }} } }\\
{\begin{array}{*{20}{c}}
{}&{}
\end{array} + {V_a}{C_l}\sum\limits_{v \in \mathcal{V}_a} {\sum\limits_{n = 1}^{{N_v}} {{{\left[ {1 - {{\rm{y}}_{vn}}({\bf{w}}_v^T({{\bf{x}}_{vn}} - {\delta _{vn}}) + {b_v})} \right]}_ + }} } }\\
{\begin{array}{*{20}{c}}
{}&{}
\end{array} - {C_a}\sum\limits_{v \in \mathcal{V}_a} {\sum\limits_{n = 1}^{{N_v}} {{{\left\| {{\delta _{vn}}} \right\|}_0}} } }
\end{array}\\
{\rm{s}}{\rm{.t}}{\rm{.   }}\begin{array}{*{20}{c}}
{\begin{array}{*{20}{l}}
{{{\bf{w}}_v} = {{\bf{w}}_u},{b_v} = {b_u},}\\
{({\delta _{v1}},...,{\delta _{v{N_v}}}) \in {{\cal U}_v},}
\end{array}}&{\begin{array}{*{20}{l}}
{\forall v \in {\cal V},u \in {\mathcal{B}_v};}\\
{\forall v \in {{\cal V}_a}.}
\end{array}}&{\begin{array}{*{20}{l}}
{(\ref{eq:alminmax}a)}\\
{(\ref{eq:alminmax}b)}
\end{array}}
\end{array}
\end{array}
\end{equation}

Note that there are two sets of constraints: (\ref{eq:alminmax}a) only contributes to the minimization part of the problem, while (\ref{eq:alminmax}b) only affects the maximization part. The first term of $K\left( {\left\{ {{{\bf{w}}_v},{b_v}} \right\},\left\{ {{\delta _{vn}}} \right\}} \right)$ is the inverse of the distance of margin. The second term is the error penalty of nodes without attacker.  The third term is the error penalty of nodes with attacker, and the last term is the cost function for the attacker. {On the one hand}, minimizing the objective function captures the trade-off between a larger margin and a small error penalty of the learner, while on the other hand, maximizing the objective function captures the trade-off between a large error penalty and a small cost of the attacker. As a result, solving Problem (\ref{eq:alminmax}) can be understood as finding the saddle point of the zero-sum game between the attacker and the learner.

\begin{definition}
\label{definitino1}
Let $\mathcal{S}_L $ and $ \mathcal{S}_A$ be the action sets for the DSVM learner and the attacker respectively. Notice that here $\mathcal{S}_A = \{ \mathcal{U}_v \}_{v\in \mathcal{V}_a}$. Then, the strategy pair $\left( {\left\{ {{{\bf{w}}_v^*},{b_v^*}} \right\},\left\{ {{\delta _{vn}^*}} \right\}} \right)$ is a saddle-point solution of the zero-sum game defined by the triple ${G_z} := \left\langle {\left\{ {L,A} \right\},\left\{ {{\mathcal{S}_L},{\mathcal{S}_A}} \right\},K} \right\rangle $, if

\[\begin{array}{l}
K\left( {\left\{ {{{\bf{w}}_v^*},{b_v^*}} \right\},\left\{ {{\delta _{vn}}} \right\}} \right) \\ \ \ \ \ \  \leq   K\left( {\left\{ {{{\bf{w}}_v^*},{b_v^*}} \right\},\left\{ {{\delta _{vn}^*}} \right\}} \right) \leq K\left( {\left\{ {{{\bf{w}}_v},{b_v}} \right\},\left\{ {{\delta _{vn}^*}} \right\}} \right),\forall v \in \mathcal{V},
\end{array} \]
where $K$ is the objective function from Problem (\ref{eq:alminmax}).
\end{definition}
 
Based on the property of the action set and atomic action set, Problem (\ref{eq:alminmax}) can be further simplified as stated in the following proposition.

\begin{proposition}
\label{proposition1}
Assume that $\mathcal{U}_v$ is an action set with corresponding atomic action set $\mathcal{U}_{v0}$. Then, Problem (\ref{eq:alminmax}) is equivalent to the following optimization problem:
\begin{equation}
\label{eq:MinMax}
\begin{array}{l}
\begin{array}{*{20}{l}}
{\mathop {\min }\limits_{\left\{ {{{\bf{w}}_v},{b_v},\{\xi_{vn}\}} \right\}} \mathop {\max }\limits_{\{ {\delta _v}\} } \frac{1}{2}\sum\limits_{v \in \mathcal{V}} {{{\left\| {{{\bf{w}}_v}} \right\|_2^2}}}  + V{C_l}\sum\limits_{v \in \mathcal{V}} {\sum\limits_{n = 1}^{{N_v}} {{\xi _{vn}}} } }\\
{\begin{array}{*{20}{c}}
{}&{}
\end{array} \ \ \ \ \ \ \ \ \ \ \ + \sum\limits_{v \in \mathcal{V}_a} {\left( {{V_a}{C_l}{\bf{w}}_v^T{\delta _v} - {C_a}{{\left\| {{\delta _v}} \right\|}_0}} \right)} }
\end{array}\\
{\rm{s}}.{\rm{t}}.\\
\begin{array}{*{20}{c}}
{\begin{array}{*{20}{l}}
{{{\rm{y}}_{vn}}({\bf{w}}_v^T{{\bf{x}}_{vn}} + {b_v}) \ge 1 - {\xi _{vn}},}\\
{{\xi _{vn}} \ge 0,}\\
{{{\bf{w}}_v} = {{\bf{w}}_u},{b_v} = {b_u},}\\
{{\delta _v} \in {\mathcal{U}_{v0}},}
\end{array}}&{\begin{array}{*{20}{l}}
{\forall v \in \mathcal{V},n = 1,...,{N_v};}\\
{\forall v \in \mathcal{V},n = 1,...,{N_v};}\\
{\forall v \in \mathcal{V},u \in {\mathcal{B}_v};}\\
{\forall v \in {\mathcal{V}_a}.}
\end{array}}
\end{array}
\end{array}
\end{equation}
\end{proposition}

\begin{proof} See Appendix A. \end{proof}

In Problem (\ref{eq:alminmax}), the third term of function $K\left( {\left\{ {{{\bf{w}}_v},{b_v}} \right\},\left\{ {{\delta _{vn}}} \right\}} \right)$ is the sum of hinge loss functions of the nodes under attack. This term is affected by the decision variables of both players. However, Problem (\ref{eq:MinMax}) transforms that into hinge loss functions without attacker's action $\delta_{vn}$ and a coupled multiplication of $\mathbf{w}_v$ and $\delta_v$. Notice that $\delta_v$ here can be seen as the combination of all the $\delta_{vn}$ in node $v$. In this way, the only coupled term is $V_a C_l \mathbf{w}_v^T \delta_v$, which is linear in the decision variables of the attacker and the learner respectively. 
\section{ADMoM-DSVM and Distributed Algorithm}
\label{sec:ADMoMDSVM}
In the previous section, we have combined Problem (\ref{eq:aDSVM1}) for the learner with Problem (\ref{eq:aDSVM}) for the attacker into one minimax Problem (\ref{eq:alminmax}), and showed its equivalence to Problem (\ref{eq:MinMax}). In this section, we develop iterative algorithms to find equilibrium solutions to Problem (\ref{eq:MinMax}). 

Firstly, we define $\mathbf{r}_v:=[{\bf{w}}_v^T,b_v]^T$, the augmented matrix $\mathbf{X}_v:=[(\mathbf{x}_{v1},...,\mathbf{x}_{vN_v})^T,\mathbf{1}_v]$, the diagonal label matrix $\mathbf{Y}_v:=diag([y_{v1},...,y_{vN_v}])$, and the vector of slack variables $\xi_v:=[\xi_{v1},....,\xi_{vN_v}]^T$. {With these definitions, it follows readily that ${\bf{w}}_v=\widehat{\mathbf{I}}_{p \times(p+1)}\mathbf{r}_v$, where $\widehat{\mathbf{I}}_{p \times(p+1)}=[{\mathbf{I}}_{p\times p},\mathbf{0}_{p\times  1}]$ is a $p \times (p+1)$ matrix with its first $p$ columns being an identity matrix, and its $(p+1)$ column being a zero vector. We also relax the $l_0$ norm to $l_1$ norm to represent the cost function of the attacker.} Thus, Problem (\ref{eq:MinMax}) can be rewritten as
\begin{equation}
\label{eq:MinMaxMatrix}
\begin{array}{*{20}{l}}
{\begin{array}{*{20}{l}}
{\mathop {\min }\limits_{\left\{ {{{\bf{r}}_v},{\xi _v},{\omega _{vu}}} \right\}} \mathop {\max }\limits_{\{ {\delta _v}\} } \frac{1}{2}\sum\limits_{v \in {\cal V}} {{\bf{r}}_v^T{\Pi _{p + 1}}{{\bf{r}}_v}}  + V{C_l}\sum\limits_{v \in {\cal V}} {{\bf{1}}_v^T{\xi _v}} }\\
{\begin{array}{*{20}{c}}
{}&{\begin{array}{*{20}{c}}
{}&{}
\end{array}}
\end{array} + \sum\limits_{v \in {{\cal V}_a}} {\left( {{V_a}{C_l}{\bf{r}}_v^T{\widehat{\bf{I}}_{p\times(p + 1)}^T} {\delta _v} - {C_a}{{\left\| {{\delta _v}} \right\|}_{{1}}}} \right)} }
\end{array}}\\
{{\rm{s}}.{\rm{t}}.\begin{array}{*{20}{l}}
{\begin{array}{*{20}{l}}
{{{\bf{Y}}_v}{{\bf{X}}_v}{{\bf{r}}_v} \ge {{\bf{1}}_v} - {{\bf{\xi }}_v},}\\
{{{\bf{\xi }}_v} \ge {{\bf{0}}_v},}\\
{{{\bf{r}}_v} = {\omega _{vu}},{\omega _{vu}} = {{\bf{r}}_u},}\\
{{\delta _v} \in {{\cal U}_{v0}},}
\end{array}}&{\begin{array}{*{20}{l}}
{\forall v \in \mathcal{V};}\\
{\forall v \in \mathcal{V};}\\
{\forall v \in \mathcal{V},\forall u \in {\mathcal{B}_v};}\\
{\forall v \in {\mathcal{V}_a}.}
\end{array}\begin{array}{*{20}{c}}
{}&{\begin{array}{*{20}{c}}
{(\ref{eq:MinMaxMatrix}a)}\\
{(\ref{eq:MinMaxMatrix}b)}\\
{(\ref{eq:MinMaxMatrix}c)}\\
{(\ref{eq:MinMaxMatrix}d)}
\end{array}}
\end{array}}
\end{array}}
\end{array}
\end{equation}
{Note that $\Pi_{p+1} = \widehat{\mathbf{I}}_{p \times(p+1)}^T\widehat{\mathbf{I}}_{p \times(p+1)}$ is a $(p+1)\times(p+1)$ identity matrix with its $(p+1,p+1)$-st entry being $0$}. $\omega_{vu}$ is used to decompose the decision variable $\mathbf{r}_v$ to its neighbors $\mathbf{r}_u$, where $u \in \mathcal{B}_v$. Problem (\ref{eq:MinMaxMatrix}) is a minimax problem with matrix form coming from Problem (\ref{eq:alminmax}). To solve Problem (\ref{eq:MinMaxMatrix}), we first prove that the minimax problem is equivalent to the max-min problem, then we use the best response dynamics for the min-problem and max-problem separately.

\begin{proposition} 
\label{proposition2}
Let $K'(\{\mathbf{r}_v,\xi_{v}\},\{\delta_v\})$ represent the objective function in Problem (\ref{eq:MinMaxMatrix}), the minimax problem 
\[\begin{array}{l}
\mathop {\min }\limits_{\left\{ {{{\bf{r}}_v},{\xi _v}} \right\}} \mathop {\max }\limits_{\left\{ {{\delta _v}} \right\}} K'(\{ {{\bf{r}}_v},{\xi _v}\} ,\{ {\delta _v}\} )\\
\begin{array}{*{20}{c}}
{{\rm{s}}{\rm{.t}}{\rm{.}}}&{{\rm{(\ref{eq:MinMaxMatrix}a),(\ref{eq:MinMaxMatrix}b),(\ref{eq:MinMaxMatrix}c),(\ref{eq:MinMaxMatrix}d).}}}
\end{array}
\end{array}\]
yields the same saddle-point equilibrium as the max-min problem
\[\begin{array}{l}
\mathop {\max }\limits_{\left\{ {{\delta _v}} \right\}} \mathop {\min }\limits_{\left\{ {{{\bf{r}}_v},{\xi _v}} \right\}} K'(\{ {{\bf{r}}_v},{\xi _v}\} ,\{ {\delta _v}\} )\\
\begin{array}{*{20}{c}}
{{\rm{s}}{\rm{.t}}{\rm{.}}}&{{\rm{ (\ref{eq:MinMaxMatrix}a),(\ref{eq:MinMaxMatrix}b),(\ref{eq:MinMaxMatrix}c),(\ref{eq:MinMaxMatrix}d). }}}
\end{array}
\end{array}\]
{
Moreover, there exists an equilibrium of the minimax or max-min Problem (\ref{eq:MinMaxMatrix}), but the equilibrium is not necessarily unique}.\end{proposition}

\begin{proof}
See Appendix B.
\end{proof} 

Proposition 2 illustrates that the minimax problem is equivalent to the max-min problem, and thus we can construct the best response dynamics for the min-problem and max-problem separately when solving Problem (\ref{eq:MinMaxMatrix}). The min-problem and max-problem are archived by fixing $\{\mathbf{r}_v,\xi_v\}$ and $\{\delta_{v}\}$, respectively. We will also show that both the min-problem and the max-problem can be solved in a distributed way.
\subsection{Max-problem for fixed $\{\mathbf{r}_v^*,\xi_v^*\}$}
For fixed $\{\mathbf{r}_v^*,\xi_v^*\}$, the first two terms of the objective function and the first three constraints in Problem (\ref{eq:MinMaxMatrix}) can be ignored as they are not related to the max-problem. We have
\begin{equation}
\label{eq:Max}
\begin{array}{l}
\mathop {\max }\limits_{\{\delta _v\}} \sum\limits_{v \in \mathcal{V}_a} {\left( {{V_a}{C_l}{\bf{r}_v^*}^T{\widehat{\bf{I}}_{p\times(p + 1)}^T}{\delta _v}} - {C_a}\left\| {{\delta _v}} \right\|_{{1}}\right)} \\
\begin{array}{*{20}{c}}
{{\rm{s}}{\rm{.t}}{\rm{.}}}&{{\delta _v} \in {{\cal U}_{v0}}},&{\forall v \in {\cal V}_a}.
\end{array}
\end{array}
\end{equation}
Note that $\delta_v$ is independent in the Problem (\ref{eq:Max}), and thus we can separate Problem (\ref{eq:Max}) into $V_a$ sub-max-problems solving which is equivalent to solving the global max-problem. {We have relaxed the $l_0$ norm to $l_1$ norm to represent the cost function of the attacker.} By writing the equivalent form of the $l_1$-norm optimization, we arrive at the following problem
\begin{equation}
\label{eq:MaxSol}
\begin{array}{l}
\mathop {\max }\limits_{\left\{ {{\delta _v},{s_v}} \right\}} {V_a}{C_l}{\bf{r}_v^*}^T{\widehat{\bf{I}}_{p\times(p + 1)}^T}{\delta _v} - {{\bf{1}}^T}{s_v}\\
{\rm{s}}.{\rm{t}}.{\rm{   }}\begin{array}{*{20}{c}}
{\begin{array}{*{20}{c}}
{{C_a}{\delta _v} \le {s_v},}\\
{{C_a}{\delta _v} \ge  - {s_v},}\\
{{\delta _v} \in {{\cal U}_{v0}},}
\end{array}}&{\begin{array}{*{20}{c}}
{\forall v \in {\cal V}_a;}\\
{\forall v \in {\cal V}_a;}\\
{\forall v \in {\cal V}_a.}
\end{array}}
\end{array}
\end{array}
\end{equation}
Problem (\ref{eq:MaxSol}) is a convex optimization problem, the objective function and the first two constraints are linear while the third constraint is convex. Note that each node can achieve their own $\delta_v$ without transmitting information to other nodes. The global Max-Problem (\ref{eq:Max}) now is solved in a distributed fashion using $V_a$ Sub-Max-Problems (\ref{eq:MaxSol}).
\subsection{Min-problem for fixed $\{\delta_v^*\}$}
For fixed $\{\delta_v^*\}$, we have
\begin{equation}
\label{eq:Min}
\begin{array}{*{20}{c}}
{\begin{array}{*{20}{l}}
{\mathop {\min }\limits_{\left\{ {{{\bf{r}}_v},{\omega _{vu}},{{\bf{\xi }}_v}} \right\}} \frac{1}{2}\sum\limits_{v \in \mathcal{V}} {{\bf{r}}_v^T{\Pi _{p + 1}}{{\bf{r}}_v}} }\\
{\begin{array}{*{20}{c}}
{}&{}
\end{array} + {V_a}{C_l}\sum\limits_{v \in \mathcal{V}_a} {{\bf{r}}_v^T{\widehat{\bf{I}}_{p\times(p + 1)}^T}\delta _v^*}  + V{C_l}\sum\limits_{v \in \mathcal{V}} {{{\bf{1}}_v^T{{\bf{\xi }}_v}} } }
\end{array}}\\
{{\rm{s}}.{\rm{t}}.\begin{array}{*{20}{c}}
{\begin{array}{*{20}{c}}
{{{\bf{Y}}_v}{{\bf{X}}_v}{{\bf{r}}_v} \ge {{\bf{1}}_v} - {{\bf{\xi }}_v}},&{\forall v \in \mathcal{V};} &{(\ref{eq:Min}a)}\\
{{{\bf{\xi }}_v} \ge {{\bf{0}}_v}},&{\forall v \in \mathcal{V};} &{(\ref{eq:Min}b)}\\
{{{\bf{r}}_v} = {\omega _{vu}},{\omega _{vu}} = {{\bf{r}}_u}},&{\forall v \in \mathcal{V},\forall u \in {\mathcal{B}_u} .} &{(\ref{eq:Min}c)}
\end{array}}
\end{array}}
\end{array}
\end{equation}
Note that term ${ - {C_a }\left\| {\delta _v^*} \right\|_{{1}}}$ is ignored since it does not play a role in the minimization problem. Furthermore, we use the alternating direction method of multipliers to solve Problem (\ref{eq:Min}). 

The surrogate augmented Lagrangian function for Problem (\ref{eq:Min}) is 
\begin{equation}
\label{eq:MinLag}
\begin{array}{l}
{L_\eta }(\{ {{\bf{r}}_v} , {\xi _v}\} ,\{ {\omega _{vu}}\} ,\{ {\alpha _{vu,k}}\} )\\
 = \frac{1}{2}\sum\limits_{v \in \mathcal{V}} {{\bf{r}}_v^T{\Pi _{p + 1}}{{\bf{r}}_v}}  + V{C_l}\sum\limits_{v \in \mathcal{V}} {{\bf{1}}_v^T{{\bf{\xi }}_v}} \\
 + {V_a}{C_l}\sum\limits_{v \in \mathcal{V}_a} {{\bf{r}}_v^T{\widehat{\bf{I}}_{p\times(p + 1)}^T}\delta _v^*} \\
 + \sum\limits_{v \in \mathcal{V}} {\sum\limits_{u \in {{\cal B}_v}} {\alpha _{vu,1}^T({{\bf{r}}_v} - {\omega _{vu}})} }  + \sum\limits_{v \in \mathcal{V}} {\sum\limits_{u \in {{\cal B}_v}} {\alpha _{vu,2}^T({\omega _{vu}} - {{\bf{r}}_u})} } \\
 + \frac{\eta }{2}\sum\limits_{v \in \mathcal{V}} {\sum\limits_{u \in {{\cal B}_v}} {{{\left\| {{{\bf{r}}_v} - {\omega _{vu}}} \right\|_2^2}}} }  + \frac{\eta }{2}\sum\limits_{v \in \mathcal{V}} {\sum\limits_{u \in {{\cal B}_v}} {{{\left\| {{\omega _{vu}} - {{\bf{r}}_u}} \right\|_2^2}}} }.
\end{array}
\end{equation}
Notice that $\alpha _{vu,1}$ and $\alpha _{vu,2}$ donate the Lagrange multipliers with respect to $\bf{r}_v = \omega _{vu}$ and $\omega _{vu} = \bf{r}_u$. ``Surrogate'' here means that $L_\eta$ does not include the constraints {(\ref{eq:Min}a)} and {(\ref{eq:Min}b)}. ``Augmented'' indicates that $L_\eta$
contains two quadratic terms which are scaled by constant $\eta > 0$, and these two terms are used to further regularize the equality constraints in (\ref{eq:Min}). ADMoM solves Problem (\ref{eq:Min}) by following update rules\cite{boyd2011distributed}:
\begin{equation}
\label{eq:MinADMMi1}
\begin{array}{l}
\left\{ {{{\bf{r}}_v^{(t+1)}},{\xi _v^{(t+1)}}} \right\}
 \in \arg \mathop {\min }\limits_{\{ {{\bf{r}}_v},{\xi _v}\} } {L_\eta }(\{ {{\bf{r}}_v} ,{\xi _v}\} ,\{ {\omega _{vu}^{(t)}}\} ,\{ {\alpha _{vu,k}^{(t)}}\} );
\end{array}
\end{equation}
\begin{equation}
\label{eq:MinADMMi2}
\begin{array}{l}
\left\{ {{\omega _{vu}^{(t+1)}}} \right\}
 \in \arg \mathop {\min }\limits_{\{\omega _{vu}\}} {L_\eta }(\{ {{\bf{r}}_v^{(t+1)}}, {\xi _v^{(t+1)}}\} ,\{ {\omega _{vu}}\} ,\{ {\alpha _{vu,k}^{(t)}}\} );
\end{array}
\end{equation}
\begin{equation}
\label{eq:MinADMMi3}
\begin{array}{l}
{\alpha _{vu,1}^{(t + 1)}} = {\alpha _{vu,1}^{(t)}} + \eta ({{\bf{r}}_v^{(t+1)}} - {\omega _{vu}^{(t+1)}}),
\forall v \in \mathcal{V},\forall u \in {\mathcal{B}_v};
\end{array}
\end{equation}
\begin{equation}
\label{eq:MinADMMi4}
\begin{array}{l}
{\alpha _{vu,2}^{(t+1)}} = {\alpha _{vu,2}^{(t)}} + \eta ( {\omega _{vu}^{(t+1)}}-{{\bf{r}}_u^{(t+1)}} ),
\forall v \in \mathcal{V},\forall u \in {\mathcal{B}_v}.
\end{array}
\end{equation}
{Note that (\ref{eq:MinADMMi1})-(\ref{eq:MinADMMi4}) contains two quadratic programming problems and two linear computations. Furthermore, (\ref{eq:MinADMMi1})-(\ref{eq:MinADMMi4}) can be simplified into the following proposition.}

\begin{proposition} 
\label{proposition3}
Each node iterates with randomly initialization $\lambda_v^{(0)},\mathbf{r}_v^{(0)}$ and $\alpha_v^{(0)} = \mathbf{0}_{(p+1)\times 1}$, 
\begin{equation}
\label{eq:MinADMMRi1}
\begin{array}{l}
{\lambda _v^{(t+1)}} \in \arg \mathop {\max }\limits_{{\bf{0}} \le {{\bf{\lambda }}_v} \le V{C_l}{{\bf{1}}_v}}  - \frac{1}{2}\lambda _v^T{{\bf{Y}}_v}{{\bf{X}}_v}{\bf{U}}_v^{ - 1}{\bf{X}}_v^T{{\bf{Y}}_v}{\lambda _v}\\
\begin{array}{*{20}{c}}
{\begin{array}{*{20}{c}}
{}&{}&{}
\end{array}}&{}&{}&{}
\end{array} + {({{\bf{1}}_v} + {{\bf{Y}}_v}{{\bf{X}}_v}{\bf{U}}_v^{ - 1}{{\bf{f}}_v^{(t)}})^{T}}{\lambda _v},
\end{array}
\end{equation}
\begin{equation}
\label{eq:MinADMMRi2}
{{\bf{r}}_v^{(t+1)}} = {\bf{U}}_v^{ - 1}\left( {{\bf{X}}_v^T{{\bf{Y}}_v}{\lambda _v^{(t+1)}} - {{\bf{f}}_v^{(t)}}} \right),
\end{equation}
\begin{equation}
\label{eq:MinADMMRi3}
{\alpha _v^{(t+1)}} = {\alpha _v^{(t)}} + \frac{\eta }{2}\sum\limits_{u \in {\mathcal{B}_v}} {\left[ {{{\bf{r}}_v}^{(t + 1)} - {{\bf{r}}_u^{(t + 1)}}} \right]}, 
\end{equation}
where $\mathbf{U}_v={\Pi_{p+1}}+2\eta\vert \mathcal{B}_v\vert\mathbf{I}_{p+1},\mathbf{f}_v^{(t)}=V_a C_l{\widehat{\mathbf{I}}_{p\times (p+1)}^T}\delta_v^*+2\alpha_v^{(t)}-\eta\sum_{u\in \mathcal{B}_v}(\mathbf{r}_v^{(t)}+\mathbf{r}_u^{(t)}),\eta>0$ .
\end{proposition}
\begin{proof} { A similar proof can be found in \cite{forero2010consensus}.
By solving (\ref{eq:MinADMMi2}) directly, we have that $\omega_{vu}^{(t+1)} = \frac{1}{2\eta}(\alpha_{vu,1}^{(t)}-\alpha_{vu,2}^{(t)})+\frac{1}{2}(\mathbf{r}_v^{(t+1)}+\mathbf{r}_u^{(t+1)})$, and thus, (\ref{eq:MinADMMi2}) can be eliminated by directly plugging the solution into (\ref{eq:MinADMMi1}), (\ref{eq:MinADMMi3}), and (\ref{eq:MinADMMi4}). }

{By plugging the solution of (\ref{eq:MinADMMi2}) into (\ref{eq:MinADMMi3}) and (\ref{eq:MinADMMi4}), we can achieve that ${\alpha _{vu,1}^{(t + 1)}} =\frac{1}{2}( {\alpha _{vu,1}^{(t)} +\alpha _{vu,2}^{(t)} }) + \frac{\eta}{2} ({{\bf{r}}_v^{(t+1)}} -{{\bf{r}}_u^{(t+1)}})$, and $ {\alpha _{vu,2}^{(t + 1)}} =\frac{1}{2}( {\alpha _{vu,1}^{(t)} +\alpha _{vu,2}^{(t)} }) + \frac{\eta}{2} ({{\bf{r}}_v^{(t+1)}} -{{\bf{r}}_u^{(t+1)}})$, respectively. Let $\alpha _{vu,1} ^{(0)} = \alpha _{vu,2} ^{(0)} = \mathbf{0}_{(p+1) \times 1}$ be the initial condition, we have that $\alpha _{vu,1} ^{(t)} = \alpha _{vu,2} ^{(t)}$ for $t\geq 0$. Thus, (\ref{eq:MinADMMi2}), (\ref{eq:MinADMMi3}), and (\ref{eq:MinADMMi4}) can be simplified further as
$\omega_{vu}^{(t+1)} =\frac{1}{2}(\mathbf{r}_v^{(t+1)}+\mathbf{r}_u^{(t+1)})$, and $\alpha _{vu,1} ^{(t+1)} = \alpha _{vu,2} ^{(t+1)}  = {\alpha _{vu,1}^{(t)}  + \frac{\eta}{2} ({{\bf{r}}_v^{(t+1)}} -{{\bf{r}}_u^{(t+1)}})= \alpha _{vu,2}^{(t)} } + \frac{\eta}{2} ({{\bf{r}}_v^{(t+1)}} -{{\bf{r}}_u^{(t+1)}})$, respectively. }

{
By plugging the solution of (\ref{eq:MinADMMi2}) into (\ref{eq:MinADMMi1}), the sixth and seventh terms of the objective function in (\ref{eq:MinADMMi1}) can be simplified as $ \eta \sum_{v \in \mathcal{V}} {\sum_{u \in {{\cal B}_v}} {{{\left\| {{{\bf{r}}_v} - \frac{1}{2}(\mathbf{r}_v^{(t)}+\mathbf{r}_u^{(t)})} \right\|_2^2}}} }$. Moreover, notice that the following equality holds for the forth and fifth terms of the objective function in (\ref{eq:MinADMMi1}):
\[\begin{array}{l}
 \sum\limits_{v \in \mathcal{V}} {\sum\limits_{u \in {{\cal B}_v}} {\alpha _{vu,1}^{(t)T}({{\bf{r}}_v} - {\omega _{vu}^{(t)}})} }  + \sum\limits_{v \in \mathcal{V}} {\sum\limits_{u \in {{\cal B}_v}} {\alpha _{vu,2}^{(t)T}({\omega _{vu}^{(t)}} - {{\bf{r}}_u})} } \\
=   \sum\limits_{v \in \mathcal{V}} {\sum\limits_{u \in {{\cal B}_v}} {\alpha _{vu,1}^{(t)T}({{\bf{r}}_v} - {{\bf{r}}_u})} } = \sum\limits_{v \in \mathcal{V}} {\sum\limits_{u \in {{\cal B}_v}} {{{\bf{r}}_v}^T(\alpha _{vu,1}^{(t)} - \alpha _{uv,1}^{(t)})} } \\
=2\sum\limits_{v \in \mathcal{V}} { {{{\bf{r}}_v}^T\sum\limits_{u \in {{\cal B}_v}}\alpha _{vu,1}^{(t)}} } = 2\sum\limits_{v \in \mathcal{V}} { {{{\bf{r}}_v}^T  \alpha_v^{(t)} } },
\end{array}\]
where $\alpha_v^{(t)} = \sum_{u \in {{\cal B}_v}}\alpha _{vu,1}^{(t)}$. Note that the first equality holds as $\alpha _{vu,1} ^{(t)} = \alpha _{vu,2} ^{(t)}$ for $t\geq 0$, the third equality holds as $\alpha _{vu,1}^{(t)} = -\alpha _{uv,1}^{(t)}$, which holds when $\alpha _{vu,1} ^{(0)} = \alpha _{vu,2} ^{(0)} = \mathbf{0}_{(p+1) \times 1}$. Thus, we only need to calculate $\alpha_v^{(t)}$ at each iteration for (\ref{eq:MinADMMi1}). As a result, (\ref{eq:MinADMMi3}) and (\ref{eq:MinADMMi4}) can be written as (\ref{eq:MinADMMRi3}).}

{
Using these results, we can rewrite Problem (\ref{eq:MinADMMi1}) as follows
\[\begin{array}{l}
\{\mathbf{r}_v^{(t+1)},\xi_v^{(t+1)}\}\in   \arg\min\limits_{ \{  \mathbf{r}_v,\xi_v \}  } \frac{1}{2}\sum\limits_{v \in \mathcal{V}} {{\bf{r}}_v^T {\Pi _{p + 1}}{{\bf{r}}_v}}  + V{C_l}\sum\limits_{v \in \mathcal{V}} {{\bf{1}}_v^T{{\bf{\xi }}_v}} 
\\ \ \ \ \ \ \ + {V_a}{C_l}\sum\limits_{v \in \mathcal{V}_a} {{\bf{r}}_v^T{\widehat{\bf{I}}_{p\times(p + 1)}^T}\delta _v^*} + 2\sum\limits_{v \in \mathcal{V}} { {{{\bf{r}}_v}^T  \alpha_v^{(t)} } }\\ \ \ \ \ \ \ + \eta \sum\limits_{v \in \mathcal{V}} {\sum\limits_{u \in {{\cal B}_v}} {{{\left\| {{{\bf{r}}_v} - \frac{1}{2}(\mathbf{r}_v^{(t)}+\mathbf{r}_u^{(t)})} \right\|_2^2}}} } \\
\begin{array}{ccc}
{\text{s.t.}}&{\begin{array}{c}
{\mathbf{Y}}_v{\mathbf{X}}_v\mathbf{r}_v \geq {\mathbf{1}}_v - \xi_v,\\ \xi_v \geq {\mathbf{0}}_v,
\end{array}}&{\begin{array}{c}
v\in \mathcal{V};\\v\in \mathcal{V}.
\end{array} }
\end{array}
\end{array}\]}

{
Let $\lambda_{v}$ and $\beta_{v}$ denote the Lagrange multipliers associated with the constraints ${\mathbf{Y}}_v{\mathbf{X}}_v\mathbf{r}_v \geq {\mathbf{1}}_v - \xi_v$ and $\xi_v \geq {\mathbf{0}}_v$, respectively. As a result, we have the Lagrange function for (\ref{eq:MinADMMi1}) as
\[\begin{array}{l}
{L}_\eta' = \frac{1}{2}\sum\limits_{v \in \mathcal{V}} {{\bf{r}}_v^T{\Pi _{p + 1}}{{\bf{r}}_v}} - \sum\limits_{v\in\mathcal{V}}\lambda_{v}^T ({\mathbf{Y}}_v{\mathbf{X}}_v\mathbf{r}_v - {\mathbf{1}}_v + \xi_v) - \sum\limits_{v\in\mathcal{V}} \beta_{v}^T\xi_v \\ \ \ \ \ \ \ + V{C_l}\sum\limits_{v \in \mathcal{V}} {{\bf{1}}_v^T{{\bf{\xi }}_v}} 
 + {V_a}{C_l}\sum\limits_{v \in \mathcal{V}_a} {{\bf{r}}_v^T{\widehat{\bf{I}}_{p\times(p + 1)}^T}\delta _v^*}\\
\ \ \ \ \ \ + 2\sum\limits_{v \in \mathcal{V}} { {{{\bf{r}}_v}^T  \alpha_v } } + \eta \sum\limits_{v \in \mathcal{V}} {\sum\limits_{u \in {{\cal B}_v}} {{{\left\| {{{\bf{r}}_v} - \frac{1}{2}(\mathbf{r}_v^{(t)}+\mathbf{r}_u^{(t)})} \right\|_2^2}}} }  .
\end{array}\]}

{
By KKT conditions, we have 
\[\begin{array}{l}
(\Pi_{p+1}+2\eta\vert \mathcal{B}_v\vert\mathbf{I}_{p+1})  \mathbf{r}_v =\mathbf{X}_v^T\mathbf{Y}_v \lambda_v^{(t+1)}  - V_aC_l{\widehat{\bf{I}}_{p\times (p + 1)}^T}\delta _v^* \\ \ \ \ \ \ \ \ -2\alpha_v^{(t)} + \eta \sum\limits_{u\in\mathcal{B}_v} (\mathbf{r}_v^{(t)} + \mathbf{r}_u^{(t)});\\
\mathbf{0} = VC_l\mathbf{1}_v - \lambda_v -\beta_v.  
\end{array}   \]
Note that $\lambda_v \geq \mathbf{0}$ and $\beta_v \geq \mathbf{0}$, thus, the second equality yields $\mathbf{0} \leq \lambda_v \leq VC_l\mathbf{1}_v$. Let $\mathbf{U}_v =\Pi_{p+1}+2\eta\vert \mathcal{B}_v\vert\mathbf{I}_{p+1} $ and $\mathbf{f}_v^{(t)} =  V_aC_l{\widehat{\bf{I}}_{p\times(p + 1)}^T} \delta _v^* +2\alpha_v^{(t)} - \eta \sum\limits_{u\in\mathcal{B}_v} (\mathbf{r}_v^{(t)} + \mathbf{r}_u^{(t)}) $, the first equality yields (\ref{eq:MinADMMRi2}). $\lambda_v$ can be achieved by solving the dual problem of Problem (\ref{eq:MinADMMi1}), which yields (\ref{eq:MinADMMRi1}).}
\end{proof}

Note that (\ref{eq:MinADMMRi1}) is a quadratic programming problem with linear inequality constraints. (\ref{eq:MinADMMRi2}) and (\ref{eq:MinADMMRi3}) are direct computations. $\mathbf{U}_v$ is a diagonal matrix. Thus, $\mathbf{U}_v^{-1}$ always exists and is easy to compute. (\ref{eq:MinADMMRi1})-(\ref{eq:MinADMMRi3}) are fully distributed iterations as each node uses their own sample data $\mathbf{X}_v $ and $\mathbf{Y}_v$. But the computations of $\mathbf{f}_v$ and $\alpha_v$ at node $v$ require the value of $\mathbf{r}_u$ form neighboring nodes. This can be achieved by allowing communications between nodes. The centralized Min-Problem (\ref{eq:Min}) can be solved in a fully distributed fashion now.
\subsection{Distributed algorithm for minimax problem}
By combining the above Proposition 3 with Problem (\ref{eq:MaxSol}), we have the method of solving Problem (\ref{eq:MinMaxMatrix}) in a distributed way as follows: The first step is that each node randomly pick an initial $ \mathbf{r}_v^{(0)}, \delta_{v}^{(0)} $ and $ \alpha_v = \mathbf{0}_{(p+1)\times 1} $, then solve Max-Problem (\ref{eq:MaxSol}) with $\{ \mathbf{r}_v^{(0)} \}$, and obtain$\{\delta_{v}^{(1)}\}$, the next step is to solve Min-Problem (\ref{eq:Min}) with $\{\delta_{v}^{(1)}\}$ using Proposition 3, and obtain $\{ \mathbf{r}_v^{(1)} \}$, then we repeat solving max-problem with $\{ \mathbf{r}_v \}$ from the previous step and min-problem with $\{\delta_{v} \}$ from the previous step until the pair $\{ \mathbf{r}_v, \delta_{v} \}$ achieves convergence. The iterations of solving Problem (\ref{eq:MinMaxMatrix}) can be summarized as follows:

\begin{proposition}
\label{Proposition 4}
With arbitrary initialization $\mathbf{\delta}_v^{(0)},\mathbf{r}_v^{(0)},\lambda_v^{(0)}$ and $\alpha_v^{(0)}=\mathbf{0}_{(p+1)\times 1}$, the iterations per node are given by:
\begin{equation}
\label{eq:MinMaxSoli1}
\begin{array}{l}
{\delta _v^{(t+1)}} \in \arg \mathop {\max }\limits_{\left\{ {{\delta _v},{s_v}} \right\}} {V_a}{C_l}{\bf{r}}_v^{(t)T} {\widehat{\bf{I}}_{p\times(p + 1)}^T}{\delta _v}\\
\begin{array}{*{20}{c}}
{\begin{array}{*{20}{c}}
{\begin{array}{*{20}{c}}
{}&{}
\end{array}}&{}
\end{array}}&{}
\end{array} \ \ \ \ \ \ \ \ \ \  - {{\bf{1}}^T}{s_v}\\
{\rm{s}}.{\rm{t}}.{\rm{   }}\begin{array}{*{20}{c}}
{\begin{array}{*{20}{l}}
{{C_a}{\delta _v} \le {s_v},}\\
{{C_a}{\delta _v} \ge  - {s_v},}\\
{{\delta _v} \in {{\cal U}_{v0}},}
\end{array}}&{\begin{array}{*{20}{l}}
{\forall v \in {\cal V}_a;}\\
{\forall v \in {\cal V}_a;}\\
{\forall v \in {\cal V}_a.}
\end{array}}
\end{array}
\end{array}
\end{equation}
\begin{equation}
\label{eq:MinMaxSoli2}
\begin{array}{*{20}{l}}
{{\lambda _v^{(t+1)}}}
{ \in \arg \mathop {\max }\limits_{{\bf{0}} \le {{\bf{\lambda }}_v} \le VC_l{{\bf{1}}_v}}  - \frac{1}{2}\lambda _v^T{{\bf{Y}}_v}{{\bf{X}}_v}{\bf{U}}_v^{ - 1}{\bf{X}}_v^T{{\bf{Y}}_v}{\lambda _v}}\\
{\begin{array}{*{20}{c}}
{}&{}
\end{array} \ \ \ \ \ \ \ \ \  \ \ \ \ \ + {{({{\bf{1}}_v} + {{\bf{Y}}_v}{{\bf{X}}_v}{\bf{U}}_v^{ - 1}{{\bf{f}}_v^{(t)}})}^T}{\lambda _v}},
\end{array}
\end{equation}
\begin{equation}
\label{eq:MinMaxSoli3}
{{{\bf{r}}_v^{(t+1)}} = {\bf{U}}_v^{ - 1}\left( {{\bf{X}}_v^T{{\bf{Y}}_v}{\lambda _v^{(t+1)}} - {{\bf{f}}_v^{(t)}}} \right)},
\end{equation}
\begin{equation}
\label{eq:MinMaxSoli4}
{{\alpha _v^{(t+1)}} = {\alpha _v^{(t)}} + \frac{\eta }{2}\sum\limits_{u \in {\mathcal{B}_v}} {\left[ {{{\bf{r}}_v^{(t+1)}} - {{\bf{r}}_u^{(t+1)}}} \right]} },
\end{equation}
where $\mathbf{U}_v={\Pi_{p+1}}+2\eta\vert \mathcal{B}_v\vert\widehat{\mathbf{I}}_{p+1},\mathbf{f}_v^{(t)}=V_a C_l{\widehat{\bf{I}}_{p\times(p + 1)}^T}\delta_v^{(t)}+2\alpha_v^{(t)}-\eta\sum_{u\in \mathcal{U}_v}(\mathbf{r}_v^{(t)}+\mathbf{r}_u^{(t)})$. 
\end{proposition}

Iterations (\ref{eq:MinMaxSoli1})-(\ref{eq:MinMaxSoli4}) are summarized into Algorithm 1. Note that at any given iteration $t$ of the algorithm, each node $v\in \mathcal{V}$ computes its own local discriminant function $g_v^{(t)} (\mathbf{x}) $ for any vector $\mathbf{x}$ as 
\begin{equation}
\label{eq:Discriminant}
{g_v^{(t)}({\bf{x}}) = [{{\bf{x}}^T},1]{{\bf{r}}_v^{(t)}}}.
\end{equation}
Algorithm 1 solves the minimax problem using ADMoM technique. It is a fully decentralized network operation, and it does not require exchanging training data or the value of decision functions, which meets the reduced communication overhead and privacy preservation requirements at the same time. The nature of the iterative algorithms also provides resiliency to the distributed machine learning algorithms. It  provides mechanisms for each node to respond to its neighbors and the adversarial behaviors in real time. When unanticipated events occur, the algorithm will be able to automatically respond and self-configure in an optimal way. Properties of Algorithm $1$ can be summarized as followings.
\begin{table}[]
\renewcommand{\arraystretch}{1.3}
\label{tablealgorithm1}
\centering
\begin{tabular}{l}
\hline
\bfseries Algorithm 1 \\
\hline

Randomly initialize $\mathbf{\delta}_v^{(0)},\mathbf{r}_v^{(0)},\lambda_v^{(0)}$ 
and $\alpha_v^{(0)}=\mathbf{0}_{(p+1)\times 1}$ \\
for every $v\in \mathcal{V}$. \\
1:\ \ \bf{for} $t=0,1,2,...$ do\\
2:\ \ \ \ \ \ \ \ \bf{for all} $v\in \mathcal{V}$ do \\
3:\ \ \ \ \ \ \ \ \ \ \ \ \ Compute $\delta_v^{(t+1)}$ via (\ref{eq:MinMaxSoli1}).\\
4:\ \ \ \ \ \ \ \ \bf{end for} \\
5:\ \ \ \ \ \ \ \ \bf{for all} $v\in \mathcal{V}$ do \\
6:\ \ \ \ \ \ \ \ \ \ \ \ \ Compute $\lambda_v^{(t+1)}$ via (\ref{eq:MinMaxSoli2}).\\
7:\ \ \ \ \ \ \ \ \ \ \ \ \ Compute $\mathbf{r}_v^{(t+1)}$ via (\ref{eq:MinMaxSoli3}).\\
8:\ \ \ \ \ \ \ \ \bf{end for} \\
9:\ \ \ \ \ \ \ \ \bf{for all} $v\in \mathcal{V}$ do \\
10:\ \ \ \ \ \ \ \ \ \ \ \ Broadcast $\mathbf{r}_v^{(t+1)}$ to all neighbors $u\in \mathcal{B}_v$.\\
11:\ \ \ \ \ \ \ \bf{end for}\\
12:\ \ \ \ \ \ \ \bf{for all} $v\in \mathcal{V}$ do \\
13:\ \ \ \ \ \ \ \ \ \ \ \ \ Compute $\alpha_v^{(t+1)}$ via (\ref{eq:MinMaxSoli4}).\\
14:\ \ \ \ \ \ \ \bf{end for}\\
15:\ \bf{end for}\\
\hline
\end{tabular}
\end{table}
\subsection{Game of Games}
The zero-sum minimax Problem (\ref{eq:MinMax}) is a global game between the two players, i.e., a learner and an attacker. The game captures the interactions on a network of $V$ nodes. However, based on the properties 
of network, the two-person zero-sum game can be treated as $V$ small games between a local learner and a local attacker. If we treat each node as a player, then the global game can be decomposed into $V$ smaller games in which each node constitutes a local game between the local learner at the node and the local attacker who attacks the node. We call this unique structure ``Game of Games''. 

To state more formally, we let the two-player zero sum game be represented by
${G_Z} = \left\langle {\left\{ {L,A} \right\},\left\{ {{\mathcal{S}_L},{\mathcal{S}_A}} \right\},K} \right\rangle , $ 
which is equivalent to a game of games defined by
${G_M}\equiv \{ G_1,G_2,...,G_{|\mathcal{V}|} \},$
where \[G_v = \left\langle {\left\{ {L_v,A_v} \right\},\left\{ {{\mathcal{S}_{L_v}},{\mathcal{S}_{A_v}}} \right\},\{\{(\ref{eq:MinMaxSoli2}),(\ref{eq:MinMaxSoli3}),(\ref{eq:MinMaxSoli4}) \}, \{(\ref{eq:MinMaxSoli1}) \}  \}} \right\rangle .\]
Notice that $G_v$ is a local game between the learner and the attacker at node $v$. $\{L_v,A_v\}$ represents two players, i.e., the DSVM learner and attacker at node $v$. The learner $L_v$ at node $v$ solves (\ref{eq:MinMaxSoli2}), (\ref{eq:MinMaxSoli3}) and (\ref{eq:MinMaxSoli4}), while the attacker $A_v$ at the node solves (\ref{eq:MinMaxSoli1}). $\{\mathcal{S}_{L_v},\mathcal{S}_{A_v}\}$ here represents the action sets for the learner and the attacker at node $v$, and $\mathcal{S}_{A_v} = \mathcal{U}_{v0} $. 
\subsection{Complexity}
Each iteration of Algorithm 1 requires the computation of 4 variables, $\delta_v$, $\lambda_v$, $\mathbf{r}_v$ and $\alpha_v$. The computation of $\delta_v$ is a convex optimization problem with a linear objective function. Calculating $\lambda_v$ requires solving a quadratic programming problem and contains an inverse of $\mathbf{U}_v$. It can be shown that the inverse of $\mathbf{U}_v$ always exists. Variables $\mathbf{r}_v$ and $\alpha_v$ are calculated directly. The complexity of the algorithm is dominated by the quadratic programming at each iteration. Since the complexity of quadratic programming is $\mathcal{O}(n^3)$, we can conclude that the complexity at each iteration is $\mathcal{O}(n^3)$. Note that the complexity of solving Problem (\ref{eq:MinMaxMatrix}) with Algorithm 1 is dominated by ADMoM, which is affected by the network topologies.
\subsection{Scalability and Real Time Property}
{
Algorithm 1 has made no assumption on the form of the datasets or the networks, and thus it is applicable to different situations. In addition, it can be implemented as a real-time algorithm as the decision variables are updated at each step. The attacker and the learner can adapt their strategies online without restarting the whole algorithm. For example, the attacker can choose to attack at any time or compromise different nodes with different capabilities; the learner can add or delete nodes, change the network connections, and add or delete the training data. The real-time property provides a convenient way for the learner to design secure network topologies and algorithms by comparing the converged saddle-point equilibrium performances under different strategies. }
\subsection{Security and Resiliency}
{
Algorithm 1 studies the situation when there is an attacker who can change the value of training data. The algorithm provides inherent security to the DSVM as it captures the attacker's goal of maximizing the classification error of the learner. The resiliency of individual nodes in the network comes from the distributed and iterative nature of the algorithm. In this algorithm, nodes in a network are cooperative. The performance of one node affects other nodes. Compromised nodes can reduce the impacts of the attacker's manipulation of training data through the information from uncompromised nodes. If a node has a sufficient number of healthy neighbors, it can learn from their classifiers to achieve an acceptable performance. However, when a large number of nodes are compromised, it will be difficult for the compromised nodes to recover from such failure. }
\subsection{Efficiency and Privacy}
{
Algorithm 1 is a fully distributed algorithm which does not require a fusion center to store or operate large datasets. In this algorithm, each node operates on their own data and computes their own discriminant functions. Thus, we can implement it efficiently compared to its centralized counterpart. Besides, Algorithm 1 only requires the communications of decision variables $\mathbf{r}_v$ rather than the training data or other parameters between the neighboring nodes, which reduces the communication overhead and keeps privacy at the same time. The notion of differential privacy can also be applied to safeguard our distributed learning framework against stronger privacy breaches. Interested readers can refer to \cite{zhang2016dual,zhang2017dynamic}.}

In next section, we will fully analyze the convergence of Algorithm $1$.
\section{Convergence of Algorithm $1$}
\label{sec:Convergence}
Convergence is important for iterative algorithms. In this section, we give a detailed proof of the convergence of our algorithm. We first prove that iterations (\ref{eq:MinADMMRi1})-(\ref{eq:MinADMMRi3}) converge to the solution of the Min-Problem (\ref{eq:Min}) for given $\{\delta_v^*\}$, then we prove iterations (\ref{eq:MinMaxSoli1})-(\ref{eq:MinMaxSoli4}) converges to the equilibrium of the minimax Problem (\ref{eq:MinMaxMatrix}). 

Since iterations (\ref{eq:MinADMMRi1})-(\ref{eq:MinADMMRi3}) come directly from (\ref{eq:MinADMMi1})-(\ref{eq:MinADMMi4}), to prove the convergence of (\ref{eq:MinADMMRi1})-(\ref{eq:MinADMMRi3}), we only need to show that iterations (\ref{eq:MinADMMi1})-(\ref{eq:MinADMMi4}) converge to the solutions of Min-Problem (\ref{eq:Min}) for given $\{\delta_v^*\}$. We will follow a similar proof in \cite{boyd2011distributed}.

Note that the Min-Problem (\ref{eq:Min}) can be reformulated with the hinge loss function as follows: 
\begin{equation}
\label{eq:MinADMMConvergenceReformulation}
\begin{array}{*{20}{l}}
\begin{array}{l}
\mathop {\min }\limits_{\{ {{\bf{r}}_v},{\omega _{vu}}\} } \frac{1}{2}\sum\limits_{v \in \mathcal{V}} {{\bf{r}}_v^T{\Pi _{p + 1}}{{\bf{r}}_v}}  + {V_a}{C_l}\sum\limits_{v \in \mathcal{V}_a} {{\bf{r}}_v^T{\widehat{\bf{I}}_{p\times(p + 1)}^T}\delta _v^*} \\
\begin{array}{*{20}{c}}
{}&{}
\end{array} + V{C_l}\sum\limits_{v \in \mathcal{V}} {\sum\limits_{n = 1}^{{N_v}} {{{\left[ {1 - {y_{vn}}[{\bf{x}}_{vn}^T,1]{{\bf{r}}_v}} \right]}_ + }} } 
\end{array}\\
{\begin{array}{*{20}{c}}
{{\rm{s}}.{\rm{t}}.}&{{{\bf{r}}_v} = {\omega _{vu}},{\omega _{vu}} = {{\bf{r}}_u},}&{\forall v \in {\cal V},\forall u \in {{\cal B}_v}.}
\end{array}}
\end{array}
\end{equation}

The objective function is convex and the constraints are all linear, and thus the min-problem is solvable, i.e., there exists a solution for the problem. The optimal value is denoted by
\begin{equation}
\label{eq:MinOptimal}
{p^*} = \inf \left\{ \begin{array}{l}
\frac{1}{2}\sum\limits_{v \in \mathcal{V}} {{\bf{r}}_v^T{\Pi _{p + 1}}{{\bf{r}}_v}}  + {V_a}{C_l}\sum\limits_{v \in \mathcal{V}_a} {{\bf{r}}_v^T{\widehat{\bf{I}}_{p\times(p + 1)}^T}\delta _v^*} \\
 + V{C_l}\sum\limits_{v \in \mathcal{V}} {\sum\limits_{n = 1}^{{N_v}} {{{\left[ {1 - {y_{vn}}[{\bf{x}}_{vn}^T,1]{{\bf{r}}_v}} \right]}_ + }} } \\
\left| \begin{array}{l}
{{\bf{r}}_v} = {\omega _{vu}},{\omega _{vu}} = {{\bf{r}}_u},
\forall v \in {\cal V},\forall u \in {{\cal B}_v}
\end{array} \right.
\end{array} \right\}.
\end{equation}

Define the unaugmented Lagrangian $L_0$ as
\begin{equation}
\label{eq:MinLagUnaugmented}
\begin{array}{l}
{L_0 }(\{ {{\bf{r}}_v}\} ,\{ {\xi _v}\} ,\{ {\omega _{vu}}\} ,\{ {\alpha _{vu,k}}\} )\\
 = \frac{1}{2}\sum\limits_{v \in \mathcal{V}} {{\bf{r}}_v^T{\Pi _{p + 1}}{{\bf{r}}_v}}  
 + {V_a}{C_l}\sum\limits_{v \in \mathcal{V}_a} {{\bf{r}}_v^T{\widehat{\bf{I}}_{p\times(p + 1)}^T}\delta _v^*} \\
   + V{C_l}\sum\limits_{v \in \mathcal{V}} {\sum\limits_{n = 1}^{{N_v}} {{{\left[ {1 - {y_{vn}}[{\bf{x}}_{vn}^T,1]{{\bf{r}}_v}} \right]}_ + }} } 
 + \sum\limits_{v \in \mathcal{V}} {\sum\limits_{u \in {{\cal B}_v}} {\alpha _{vu,1}^T({{\bf{r}}_v} - {\omega _{vu}})} } \\ + \sum\limits_{v \in \mathcal{V}} {\sum\limits_{u \in {{\cal B}_v}} {\alpha _{vu,2}^T({\omega _{vu}} - {{\bf{r}}_u})} }  .
\end{array}
\end{equation}

To show the convergence, we state the following assumption:

\begin{assumption}
\label{assumption 1}
The unaugmented Lagrangian $L_0$ has a saddle point. Explicitly, there exist $(\{ {{\bf{r}}_v}\}^* ,\{ {\xi _v}\}^* ,\{ {\omega _{vu}}\}^* ,\{ {\alpha _{vu,k}}\}^* )$ not necessarily unique, for which
\[\begin{array}{l}
{L_0}({\{ {{\bf{r}}_v}\} ^ * },{\{ {\xi _v}\} ^ * },{\{ {\omega _{vu}}\} ^ * },\{ {\alpha _{vu,k}}\} )\\
 \le {L_0}({\{ {{\bf{r}}_v}\} ^ * },{\{ {\xi _v}\} ^ * },{\{ {\omega _{vu}}\} ^ * },{\{ {\alpha _{vu,k}}\} ^ * })\\
 \le {L_0}(\{ {{\bf{r}}_v}\} ,\{ {\xi _v}\} ,\{ {\omega _{vu}}\} ,{\{ {\alpha _{vu,k}}\} ^ * }) .
\end{array}\]
\end{assumption}

From Assumption $1$, $L_0 (\{ {{\bf{r}}_v}\}^* ,\{ {\xi _v}\}^* ,\{ {\omega _{vu}}\}^* ,\{ {\alpha _{vu,k}}\}^* )$ is finite for any saddle point $(\{ {{\bf{r}}_v}\}^* ,\{ {\xi _v}\}^* ,\{ {\omega _{vu}}\}^* ,\{ {\alpha _{vu,k}}\}^* )$ . This indicates that $(\{ {{\bf{r}}_v}\}^* ,\{ {\xi _v}\}^* ,\{ {\omega _{vu}}\}^*  )$ is a solution of (\ref{eq:MinADMMConvergenceReformulation}).  Also it shows that $\{ {\alpha _{vu,k}}\}^* $ is dual optimal, and strong duality holds, i.e., the optimal values of the primal and dual problems are equal. Notice that there is no assumption on $\mathbf{X}_v$, $\mathbf{Y}_v$.

Define primal residuals $r_{vu,1}=\mathbf{r}_v-\omega_{vu}$ and $r_{vu,2}=\omega_{vu}-\mathbf{r}_u$, dual residuals $s_{vu}^{(t)} = \eta (\omega_{vu} ^{(t)} - \omega_{vu} ^{(t-1)})$, and Lyapunov function $J^{(t)}$ for the algorithm,
\begin{equation}
\label{eq:ConSumNormResiduals}
\begin{array}{l}
{J^{(t)}} = \eta \sum\limits_{v \in \mathcal{V}} {\sum\limits_{u \in {{\cal B}_v}} {\left\| {\omega _{vu}^{(t)} - \omega _{vu}^*} \right\|_2^2} }  + \frac{1}{\eta }\sum\limits_{v \in \mathcal{V}} {\sum\limits_{u \in {{\cal B}_v}} {\left\| {\alpha _{vu,1}^{(t)} - \alpha _{vu,1}^ * } \right\|_2^2} } \\
\begin{array}{*{20}{c}}
{}&{}
\end{array} + \frac{1}{\eta }\sum\limits_{v \in \mathcal{V}} {\sum\limits_{u \in {{\cal B}_v}} {\left\| {\alpha _{vu,2}^{(t)} - \alpha _{vu,2}^ * } \right\|_2^2} } .
\end{array}
\end{equation}
$J^{(t)}$ is nonnegative and it decreases in each iteration.

\begin{lemma}
\label{lemma 1} 
Under Assumption 1, the following inequalities hold:
\begin{equation}
\label{eq:ConInequ1}
\begin{array}{l}
{J^{(t+1)}} \le {J^{(t)}} - \eta \sum\limits_{v \in \mathcal{V}} {\sum\limits_{u \in {{\cal B}_v}} {\left\| {r_{vu,1}^{(t+1)}} \right\|_2^2} } - \eta \sum\limits_{v \in \mathcal{V}} {\sum\limits_{u \in {{\cal B}_v}} {\left\| {r_{vu,2}^{(t+1)}} \right\|_2^2} }\\
 - 2 \eta \sum\limits_{v \in \mathcal{V}}{\sum\limits_{u \in {{\cal B}_v}} {\left\| {\omega _{vu}^{(t+1)} - \omega _{vu}^{(t)}} \right\|_2^2} }.
\end{array}
\end{equation}
\begin{equation}
\label{eq:ConInequ2}
\begin{array}{l}
{p^{(t+1)}} - {p^ * } \le 
 - \sum\limits_{v \in \mathcal{V}} {\sum\limits_{u \in {{\cal B}_v}} {\left( {\alpha _{vu,1}^{(t+1)T}r_{vu,1}^{(t+1)} + \alpha _{vu,2}^{(t+1)T}r_{vu,2}^{(t+1)}} \right)} } \\
 + \sum\limits_{v \in \mathcal{V}} {\sum\limits_{u \in {{\cal B}_v}} {\left( {\eta (\omega _{vu}^{(t+1)} - \omega _{vu}^{(t)})(\omega _{vu}^ *  - \omega _{vu}^{(t+1)} - r_{vu,1}^{(t+1)})} \right)} } \\
 + \sum\limits_{v \in \mathcal{V}} {\sum\limits_{u \in {{\cal B}_v}} {\left( {\eta (\omega _{vu}^{(t+1)} - \omega _{vu}^{(t)})(\omega _{vu}^ *  - \omega _{vu}^{(t+1)} + r_{vu,2}^{(t+1)})} \right)} } .
\end{array}
\end{equation}
\begin{equation}
\label{eq:ConInequ3}
{p^*} - {p^{(t+1)}} \le \sum\limits_{v \in \mathcal{V}} {\sum\limits_{u \in {{\cal B}_v}} {\alpha _{vu,1}^{ * T}r_{vu,1}^{(t+1)}} }  + \sum\limits_{v \in \mathcal{V}} {\sum\limits_{u \in {{\cal B}_v}} {\alpha _{vu,2}^{ * T}r_{vu,2}^{(t+1)}} } .
\end{equation}
\end{lemma}

\begin{proof} See Appendix C. \end{proof}

Inequality (\ref{eq:ConInequ1}) indicates that $J^{(t)}$ decreases at each step, since $J^{(t)}$ is nonnegative, thus $J^{(t)}$ converges to $0$, which also indicates that $r_{vu,k}^{(t)}$ and ${\omega _{vu}^{(t+1)} - \omega _{vu}^{(t)}}$ converge to $0$. As a result, right hand sides of (\ref{eq:ConInequ2}) and (\ref{eq:ConInequ3}) converge to $0$. Since $p^{(t+1)} - p^*$ is both upper and lower bounded by values converging to $0$, $p^{(t+1)} - p^*$ converges to $0$. From these inequalities, we arrive at the following proposition.

\begin{proposition}
\label{proposition 5}
Under on Assumption 1, (\ref{eq:MinADMMi1})-(\ref{eq:MinADMMi4}) satisfy that
\begin{itemize}
\item Primal residuals $r_{vu,k}^{(t)} \rightarrow 0$ as $t \rightarrow \infty$.
\item Dual residuals
$s_{vu}^{(t)} \rightarrow 0$ as $t \rightarrow \infty$.
\item Objective $p^{(t)} \rightarrow p^*$ as $t\rightarrow \infty$.
\item Dual variables $\alpha_{vu,k}^{(t)}\rightarrow \alpha_{vu,k}^*$ as $t\rightarrow \infty$.
\item Decision variables $\mathbf{r}_v \rightarrow \mathbf{r}_v^*$.
\end{itemize}
The convergence of iterations (\ref{eq:MinADMMi1})-(\ref{eq:MinADMMi4}) to solutions of Min-Problem (\ref{eq:Min}) for given $\{\delta_v^*\}$ is guaranteed. 
\end{proposition}
\begin{proof} Under Assumption $1$, inequalities (\ref{eq:ConInequ1}), (\ref{eq:ConInequ2}) and (\ref{eq:ConInequ3}) hold based on Appendix C. 

Inequality (\ref{eq:ConInequ1}) indicates that $J^{(t)}$ decreases based on the sum of the norm of primal residuals and the change of $\omega_{vu}$ over one iteration. Since $J^{(t)} \le J^0$, $\alpha_{vu,1}^{(t)},\alpha_{vu,2}^{(t)}$ and $\omega_{vu}^{(t)}$ are bounded. By iterating (\ref{eq:ConInequ1}), we have
\[\begin{array}{l}
\eta \sum\limits_{t = 0}^\infty  {\sum\limits_{v \in \mathcal{V}} {\sum\limits_{u \in {{\cal B}_v}} {\left( {\left\| {r_{vu,1}^{(t+1)}} \right\|_2^2 + \left\| {r_{vu,2}^{(t+1)}} \right\|_2^2} \right)} } } \\
 + 2\eta \sum\limits_{t = 0}^\infty  {\sum\limits_{v \in \mathcal{V}} {\sum\limits_{u \in {{\cal B}_v}} {\left\| {\omega _{vu}^{(t+1)} - \omega _{vu}^{(t)}} \right\|_2^2} } }  \le {J^0}.
\end{array}\]
This implies that $r_{vu,1}^{(t)} , r_{vu,2}^{(t)}$  and $(\omega_{vu}^{(t+1)}-\omega_{vu}^{(t)}) $ converge to $0$ as $t \rightarrow \infty$. Thus, the dual residuals $s_{vu}^{(t)} = \eta (\omega_{vu}^{(t)}-\omega_{vu}^{(t-1)})$ converge to $0$. 

The right hand side of inequality (\ref{eq:ConInequ2}) goes to $0$, since $(\omega_{vu}^*-\omega_{vu}^{(t+1)})$ is bounded, and $r_{vu,1}^{(t)} , r_{vu,2}^{(t)}$  and $(\omega_{vu}^{(t+1)}-\omega_{vu}^{(t)}) $ converge to $0$. The right hand side of inequality (\ref{eq:ConInequ3}) goes to $0$, since $r_{vu,1}^{(t)}$ and $ r_{vu,2}^{(t)}$ converge to $0$. As a result, we have $\mathop {\lim }\limits_{t \to \infty } {p^{(t)}} = {p^ * }$, and arrive at Proposition 5. \end{proof}

Based on Proposition 5, the convergence of (\ref{eq:MinADMMi1})-(\ref{eq:MinADMMi4}) is guaranteed. Since under Assumption 1, strong duality holds, (\ref{eq:MinADMMi1})-(\ref{eq:MinADMMi4}) will converge to the solution of the Min-Problem (\ref{eq:Min}) with given $\{\delta_v^*\}$. As a result, (\ref{eq:MinADMMRi1})-(\ref{eq:MinADMMRi3}) will also convege to Problem (\ref{eq:Min}). Next we prove the convergence of Algorithm 1 to the solution of minimax Problem (\ref{eq:MinMaxMatrix}). 

Assume that the minimax Problem (\ref{eq:MinMaxMatrix}) will reach an equilibrium $\{p^*,q^*\}$, where $p^*$ is the optimal objective of the Min-Problem (\ref{eq:Min}) at the equilibrium, and $q^*$ is the optimal objective of the Max-Problem (\ref{eq:Max}) at the equilibrium. We arrive at the following result.

\begin{proposition}
\label{proposition 6}
Under on Assumption 1, (\ref{eq:MinMaxSoli1})-(\ref{eq:MinMaxSoli4}) satisfy that
\begin{itemize}
\item Pair objectives $\{p^{(t)},q^{(t)}\} \rightarrow \{p^*,q^*\}$ as $t\rightarrow \infty$.
\item Pair decision variables $\{ \mathbf{r}_v^{(t)},\delta_v^{(t)}\} \rightarrow \{ \mathbf{r}_v^*,\delta_v^*\}$ as $t\rightarrow \infty$.
\end{itemize}
In other words, the pair of objectives and decision variables will converge to the saddle-point equilibrium.
\end{proposition}
\begin{proof} In Proposition $2$, we have shown that with min-problem for the learner and the max-problem for the attacker, the constructed minimax problem is equivalent to the max-min problem. Thus, there exists an equilibrium $\{p^*,q^*\}$ with corresponding $\mathbf{r}_v^*$ and $\delta_v^*$. Since Proposition $5$ indicates that the min-problem always converges to the best response of the max-problem, with the max-problem being a linear programming problem. Therefore, we can conclude that $\{p^{(t)},q^{(t)}\} \rightarrow \{p^*,q^*\}$, $\delta_v^{(t)} \rightarrow \delta_v^*$ and $\mathbf{r}_v^{(t)} \rightarrow \mathbf{r}_v^*$. Hence, Proposition $6$ holds.  \end{proof}

Proposition $6$ shows that the separated min-problem and max-problem converge to the equilibrium, thus the convergence of Algorithm $1$ is guaranteed. With Assumption $1$, Algorithm $1$ will converge to the saddle-point solution of minimax Problem (\ref{eq:MinMaxMatrix}). Note that here we have made no assumption on $\mathbf{X}_v$ and $\mathbf{Y}_v$. Therefore, Algorithm $1$ is applicable to various different datasets. 
\section{Numerical Experiments}
\label{sec:Num}
In this section, we present numerical experiments of DSVM under adversarial environments. We use empirical risk to measure the performance of DSVM. The empirical risk at node $v$ at step $t$ is defined as follows:
\begin{equation}
\label{eq:ErrorInNodev}
{{\bf{R}}_v^{(t)}}: = \frac{1}{{{\widetilde N_{v}}}}\sum\limits_{n = 1}^{{\widetilde N_{v}}} {\frac{1}{2}\left| {{{\widetilde y}_{vn}} - {{\widehat y}_{vn}^{(t)}}} \right|} ,
\end{equation} 
where ${{{\widetilde y}_{vn}}}$ is the true label, ${{{\widehat y}_{vn}^{(t)}}}$ is the predicted label and $\widetilde N_v$ represents the number of testing samples in node $v$. The empirical risk (\ref{eq:ErrorInNodev}) sums over the number of misclassified samples in node $v$, and then divides it by the number of all testing samples in node $v$. Notice that testing samples can vary for different nodes. In order to measure the global performance, we use the global empirical risk defined as follows:   
\begin{equation}
\label{eq:Error}
{{\bf{R}}_{G}^{(t)}}: = \frac{1}{{\widetilde N}}\sum\limits_{v \in \mathcal{V}} {\sum\limits_{n = 1}^{{{\widetilde N}_v}} {\frac{1}{2}\left| {{{\widetilde y}_{vn}} - {{\widehat y}_{vn}^{(t)}}} \right|} } ,
\end{equation} 
where ${\widetilde N} = \sum\limits_{v \in \mathcal{V}} {\widetilde N_v}$, representing the total number of testing samples. Clearly, a higher global empirical risk shows that there are more testing samples being misclassified, i.e., a worse performance of DSVM. We use the first experiment to illustrate the significant impact of the attacker.

Consider a fully connected network with $3$ nodes. Each node contains $80$ training samples and $1000$ testing samples from the same global {``Rand"} dataset which are shown as points and stars in Fig. \ref{fig:NumericalExample}(a). Yellow stars and magenta points are labelled as $-1$ and $+1$, respectively. {They are generated from two-dimensional Gaussian distributions with mean vectors $(1,1)$ and $(3,3)$, and they have the same covariance matrix $\left(\begin{array}{cc}
{1}&{0} \\ 
{0}&{1}
\end{array} \right)$.} The learner has the ability $C_l = 1$ and $\eta = 1$. {The attacker attacks all three nodes with $\sum_{v\in\mathcal{V}} C_{v,\delta}= 7500 $ and $C_a = 1$.} The attack starts from the beginning of the training process. The discriminant functions found by the learner under different situations are represented by lines in Fig. \ref{fig:NumericalExample}(a). Numerical results are shown in Fig. \ref{fig:NumericalExample}(b). We can see that when there is an attacker, both the DSVM and centralized SVMs fail to separate two datasets in Fig. \ref{fig:NumericalExample}(a). {In addition, the DSVM under the control of the same attacker show worse performances as the risk is higher in Fig. \ref{fig:NumericalExample}(b). Thus, the DSVM is more vulnerable when the attacker compromises the whole system. }
\begin{figure}[]
\centering
\subfigure[]{
\includegraphics[width=0.3\textwidth]{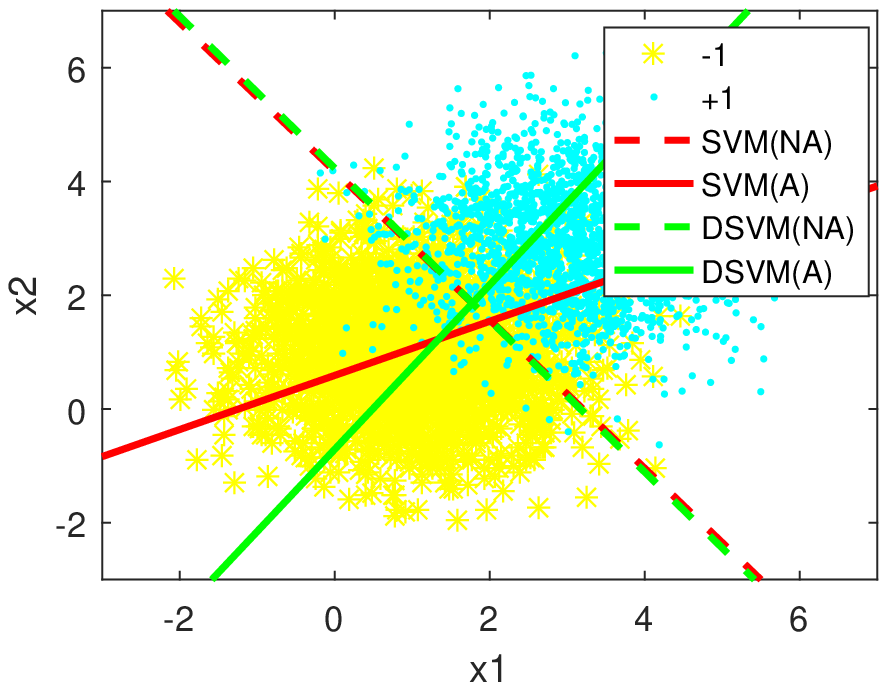}}
\subfigure[]{
\includegraphics[width=0.3\textwidth]{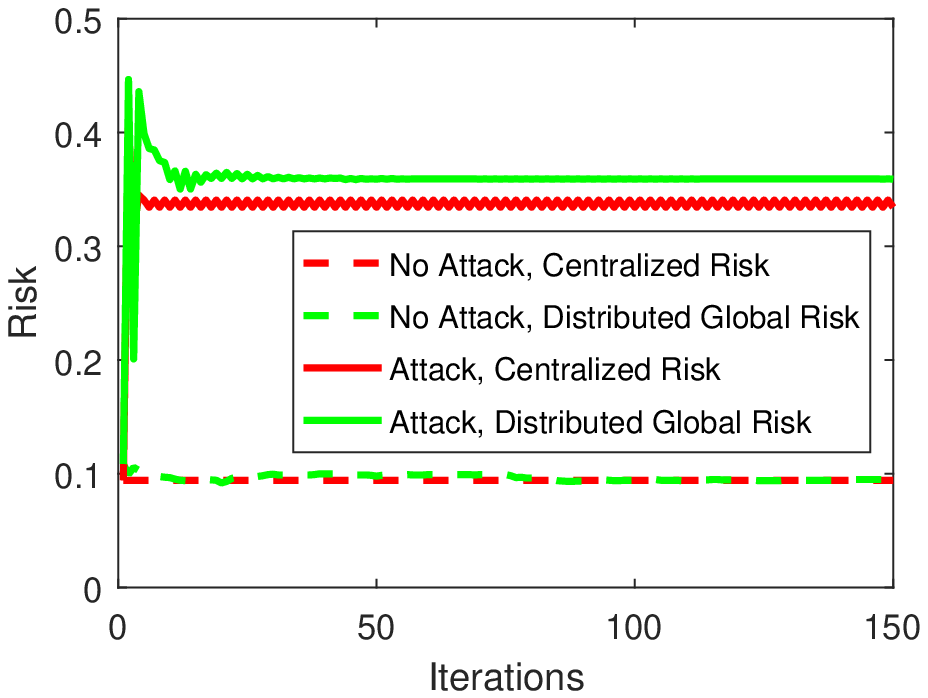}}
\vspace{-3mm}
\caption{{Evolution of the empirical risks of DSVM with an attacker at a fully connected network of $3$ nodes. Training data and testing data are generated from two Gaussian classes. The attacker attacks all three nodes from the beginning of the training process. Dotted lines and solid lines show the results for the case without an attacker and the one with an attacker. Red and green lines show the results of centralized SVMs and DSVM, respectively. } }
\label{fig:NumericalExample}
\end{figure} 

It is obvious that the attacker can cause disastrous results for the learner. In the following experiments, we will illustrate in detail how the attacker affects the training process with different abilities. We will study how the network topologies and the number of samples at each node affect the attacker's objective. {We will use the convergent global equilibrium risks to capture the impacts of the attacker on the learner. The convergence here is defined by that the moving average of the global risk with a window length of $40$ steps changes less than $0.00001$}. Without loss of generality, we will use $C_l = 1$ and $\eta = 1$ in all the experiments. {Besides the ``Rand" dataset, we will also use ``Spam"\cite{Spambase} and ``MNIST"\cite{lecun1998gradient} datasets to illustrate the impacts of the attacker. ``Spam" and ``MNIST" datasets have been widely used, for example, \cite{moreno2012study,davtalab2014multi}, and \cite{shao2014learning,tang2016extreme}, respectively. For the ``MNIST" dataset, we consider only the binary classification problem of digits ``2" and ``9".}
\subsection{Attacker's Ability}
Attacker's ability plays an important role in the game between the learner and the attacker, as a more capable attacker can modify more training data and control more nodes. There are four measures to represent the attacker' ability, the time $t$ for the attacker to take an action, the atomic action set parameter $C_{v,\delta}$, the cost parameter $C_a$, and the number of compromised nodes $| \mathcal{V}_a |$. The time $t$ for attacker to take an action will affect the results since attacking at the start of the training process is different from attacking after the convergence of the training process. $C_{v,\delta}$ comes directly from the attacker's atomic action set $\mathcal{U}_{v0}$ defined in Section III. A larger $C_{v,\delta}$ indicates that the attacker can modify the training data with a larger number. Without loss of generality, in the following experiments, we assume that the attacker has the same $C_{v,\delta}$ at all the compromised nodes, and thus we use $C_\delta = C_{v,\delta},\forall v \in \mathcal{V}_a$. $C_a$ denotes the parameter of the attacker's cost function. A larger $C_a$ will restrict the attacker's actions to change data. The number of compromised nodes $| \mathcal{V}_a |$  will affect the results as attacking more nodes gives the attacker access to modify more training samples.
\begin{figure}[]
\centering
\subfigure{
\includegraphics[width=0.3\textwidth]{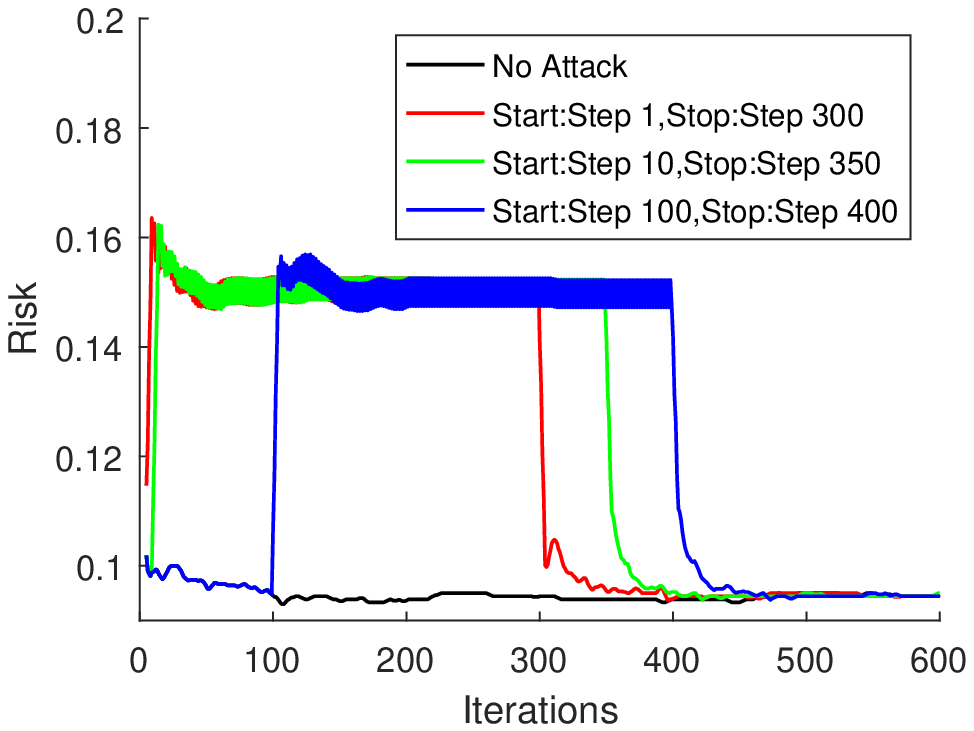}}
\subfigure{
\includegraphics[width=0.3\textwidth]{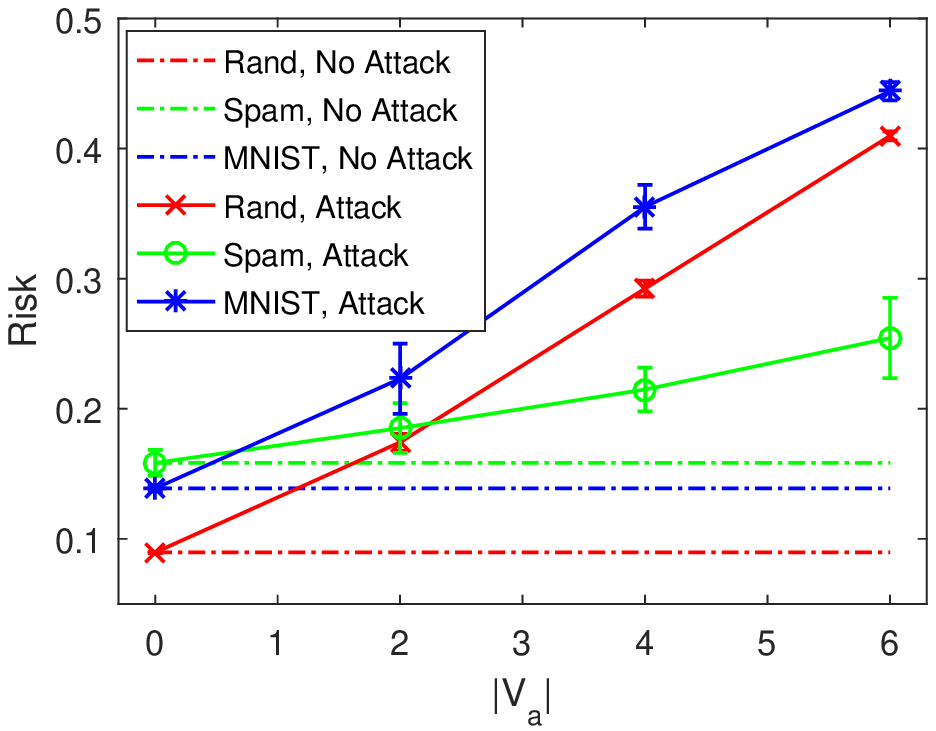}}
\vspace{-3mm}
\caption{{Global risks of DSVM at a fully connected network with $6$ nodes. Each node contains $40$ training samples and $300$ testing samples. The left figure shows the evolution of the risks on ``Rand" dataset when the attacker only attacks $1$ node, but with different starting and stopping times. The attacker has parameters $C_\delta = 10^8$ and $C_a = 0.01$. The right figure shows the average global equilibrium risks when the attacker attacks different numbers of nodes at the beginning of the training process with the ability $C_a = 1$ and $ C_\delta= 10^4$, $10^{9}$, and $10^4$ for ``Rand", ``Spam", and ``MNIST" datasets, respectively. .}}
\label{fig:TimeNode}
\end{figure} 
\begin{figure}[]
\centering
\subfigure{
\includegraphics[width=0.3\textwidth]{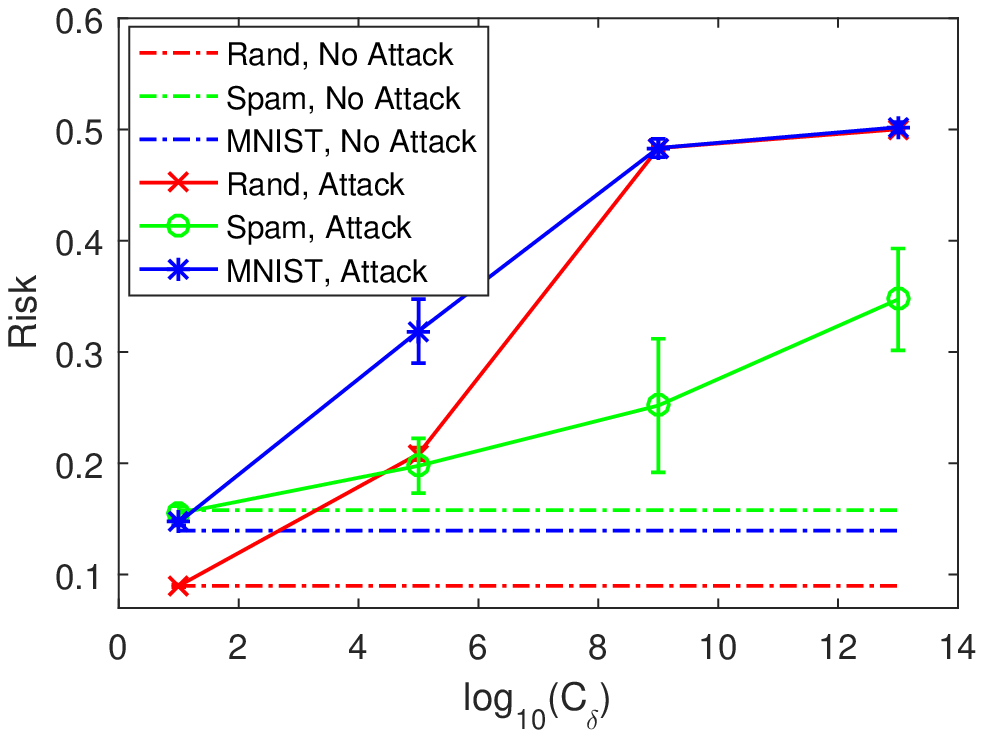}}
\subfigure{
\includegraphics[width=0.3\textwidth]{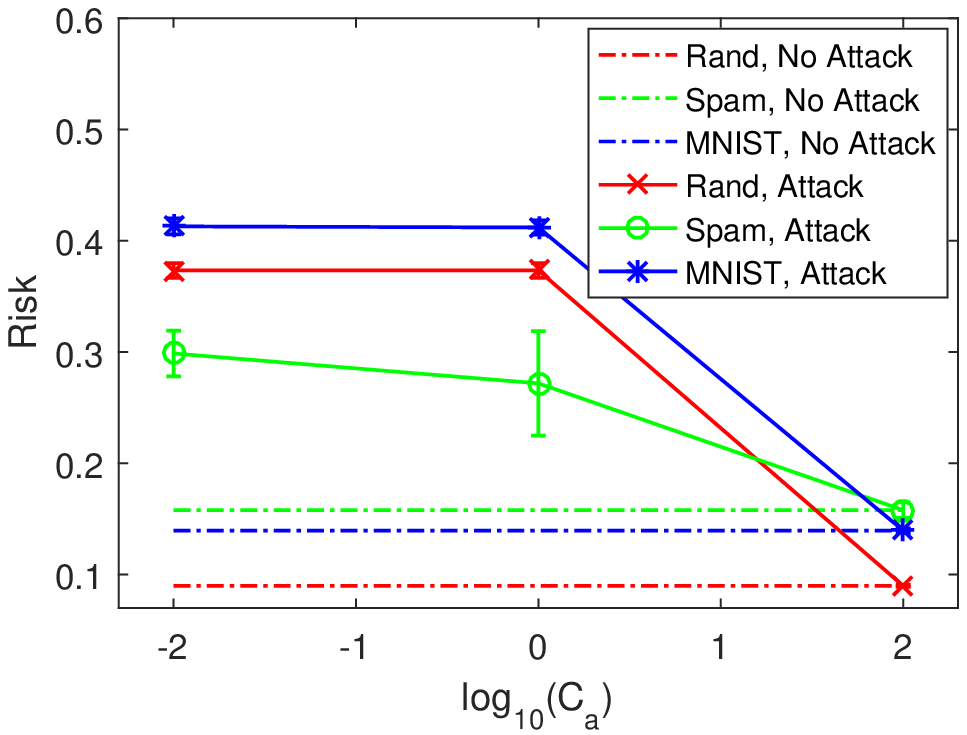}}
\vspace{-3mm}
\caption{{Average global equilibrium risks of DSVM at a fully connected network with $3$ nodes. Each node contains $80$ training samples and $600$ testing samples. The left figure shows the results with respect to $\log_{10}(C_\delta)$ when the attacker attacks $2$ nodes with $C_a = 1$. The right figure shows the results with respect to $\log_{10}(C_a)$ when the attacker attacks $1$ node with $C_\delta =10^7$.}}
\label{fig:C_delta_C_a}
\end{figure} 

{
The left figure of Fig. \ref{fig:TimeNode} shows the results of global risks when the attacker starts and stops attacking at different times. It is clear that after the attacker starts attacking, the risks quickly increase, but after the attacker stops attacking, the risks slowly come back to the level when there is no attacker. Thus, DSVM has the ability to react in real time, and it is a resilient algorithm. Moreover, though the acting times of the attacker are different, the equilibrium risks are close. As a result, we can conclude that the timing of the attacks does not significantly affect the equilibrium risks. The right figure of Fig. \ref{fig:TimeNode} shows the results of the average global equilibrium risks when the attacker attacks different numbers of nodes. It can be seen that the risks are higher when the attacker attacks more nodes, which indicates that the attacker has more influence on the learner. }

{
Fig. \ref{fig:C_delta_C_a} shows the average global equilibrium risks with respect to {$\log_{10}(C_\delta)$ and $\log_{10}(C_a)$}. We can see from the left figure that as $C_\delta$ increases, the risks become higher, which indicates that the attacker has a more significant impact on the learner. Notice that when $C_\delta$ is small, the risks are close to the risks of the case when there is no attacker, showing that the attacker has no influence on the learner as he is only capable of making small changes. From the right figure, we can see that as $C_a$ increases, the risks become lower, which indicates that the attacker is more restricted to take actions when $C_a$ is high. }

\subsection{Network Topology}
In this subsection, we will study the effect of the network topology on the game between the learner and the attacker. Consider a network with $\mathcal{V}:=\{1,2,...,V\}$ representing the set of nodes. Neighboring nodes of node $v$ are represented by set $\mathcal{B}_v \subseteq \mathcal{V}$.  We define the normalized degree of node $v$ as $\mid \mathcal{B}_v \mid / (\mid \mathcal{V} \mid -1)$, i.e., the actual number of neighbors of this node divided by the most achievable number of neighbors. The normalized degree of a node is always larger than $0$ and less or equal to $1$. A higher normalized degree indicates that the node has more neighbors. Notice that the normalized degree of a node cannot be $0$ as there is no isolated nodes in this network. In addition, we define the degree of a network as the average of normalized degrees of all the nodes. In the following experiments, all the nodes in a specified network contain the same number of training and testing samples, and the attacker has the same $C_{v,\delta}=C_{\delta}$ in all compromised nodes. {Note that we repeat each experiments $20$ times using different sets of samples at each node to find the average global equilibrium risks.}  

{\begin{table}[]
\centering
\caption{{Average global equilibrium risks $(\%)$ of DSVM at balanced networks, i.e., each node has the same numbers of neighbors, training samples, and testing samples. Each network contains the same $180$ training samples and $1800$ testing samples. The attacker compromises all the nodes with the ability $C_a = 0.001$ and $\sum_{v\in\mathcal{V}} C_\delta= 10^5$, $10^{12}$, and $10^5$ for ``Rand", ``Spam", and ``MNIST" datasets, respectively. By fixing $\sum_{v\in\mathcal{V}} C_\delta$, the attacker has the same ability in different networks. Note that ``C" indicates ``Centralized", ``D" indicates ``Degree of the network", ``NA" indicates ``No attack", and ``A" indicates ``Attack".}}
\label{Tab:NetworkGeneral}
\tabcolsep=0.11cm
{
\begin{tabular}{|l|l|l|l|l|l|}
\hline
\multicolumn{2}{|l|}{Network} & 1Node C  & 3Nodes D1  & 6Nodes D0.4 & 6Nodes D1  \\ \hline
\multirow{2}{*}{Rand}   & NA  & 8.44$\pm$0.00 & 8.46$\pm$0.00  & 8.48$\pm$0.01  & 8.44$\pm$0.00 \\ \cline{2-6} 
                        & A   & 41.74$\pm$0.00 & 42.15$\pm$0.24 & 44.20$\pm$0.33  & 43.59$\pm$0.39 \\ \hline
\multirow{2}{*}{Spam}   & NA  & 16.44$\pm$0.00 & 16.87$\pm$2.00 & 17.86$\pm$2.72  & 17.28$\pm$2.50 \\ \cline{2-6} 
                        & A   & 37.09$\pm$0.00 & 43.60$\pm$2.03 & 46.55$\pm$1.63  & 45.71$\pm$1.57 \\ \hline
\multirow{2}{*}{MNIST}  & NA  & 14.94$\pm$0.00 & 15.03$\pm$0.30 & 15.16$\pm$0.52  & 14.99$\pm$0.31 \\ \cline{2-6} 
                        & A   & 44.32$\pm$0.00 & 45.26$\pm$0.62 & 46.85$\pm$0.36  & 46.34$\pm$0.57 \\ \hline
\end{tabular}}
\end{table}}

{\begin{table}[]
\centering
\caption{{Average global equilibrium risks $(\%)$ of DSVM at networks with $6$ nodes. In the unbalanced network, Node $1$ has $5$ neighbors, while Nodes $2-6$ have $1$ neighbor. In the balanced network, Nodes $1-4$ have $2$ neighbors, while Nodes $5-6$ have $1$ neighbor. Note that both networks have the same degree $0.33$. Each node in both networks contains $40$ training samples and $300$ testing samples. The attacker attacks either the higher degree Node 1 or the lower degree Node 6 with the ability $C_a = 0.01$ and $C_\delta = 10^8$, $10^{14}$, and $10^{8}$ for ``Rand", ``Spam", and ``MNIST" datasets, respectively. }}
\tabcolsep=0.11cm
\label{Tab:NetworkIndividual}
{
\begin{tabular}{|l|l|l|l|l|l|}
\hline
\multicolumn{2}{|l|}{Network}         & \multicolumn{2}{l|}{Unbalanced} & \multicolumn{2}{l|}{Balanced}   \\ \hline
\multicolumn{2}{|l|}{Attack}          & Node 1         & Node 6         & Node 1          & Node 6        \\ \hline
\multirow{2}{*}{Rand} & NA & \multicolumn{2}{l|}{8.51$\pm$0.15}  & \multicolumn{2}{l|}{8.51$\pm$0.24}  \\ \cline{2-6} 
                          & A    & 43.43$\pm$0.22     & 35.45$\pm$0.37 & 36.00$\pm$0.46      & 28.12$\pm$0.35    \\ \hline
\multirow{2}{*}{Spam} & NA & \multicolumn{2}{l|}{16.56$\pm$0.94} & \multicolumn{2}{l|}{16.41$\pm$1.06} \\ \cline{2-6} 
                          & A    & 43.44$\pm$2.77      & 39.87$\pm$2.38    & 41.65$\pm$2.68 & 36.62$\pm$1.11      \\ \hline
\multirow{2}{*}{MNIST}& NA & \multicolumn{2}{l|}{13.84$\pm$0.21} & \multicolumn{2}{l|}{13.81$\pm$0.29} \\ \cline{2-6} 
                          & A    & 46.72$\pm$0.57     & 41.91$\pm$1.84 & 43.41$\pm$0.93      & 33.48$\pm$0.98    \\ \hline
\end{tabular}}
\end{table}

{\begin{table}[]
\centering
\caption{{Average global equilibrium risks $(\%)$ of DSVM in a fully connected network with $3$ nodes. In either case, Node $2$ and Node $3$ are fixed as either of them contains $50$ training samples, but Node $1$ contains $50$ or $100$ training samples. The attacker attacks Node $1$ with the ability $C_a = 1$, and $C_\delta = 10^6,10^{12}$, and $10^5$ for ``Rand", ``Spam", and ``MNIST" datasets, respectively. }}
\tabcolsep=0.11cm
\label{Tab:Network3Weight}
{
\begin{tabular}{|l|l|l|l|}
\hline
\multicolumn{2}{|l|}{\begin{tabular}[c]{@{}l@{}}Number of  Training \\ Samples in Each Node\end{tabular}} & 50, 50, 50     & 100, 50, 50    \\ \hline
\multirow{2}{*}{Rand}                                         & NA                                        & 9.00$\pm$0.00     & 8.51$\pm$0.00     \\ \cline{2-4} 
                                                              & A                                         & 33.70$\pm$0.58     & 28.17$\pm$0.68     \\ \hline
\multirow{2}{*}{Spam}                                         & NA                                        & 16.06$\pm$0.88 & 15.98$\pm$0.70   \\ \cline{2-4} 
                                                              & A                                         & 37.78$\pm$4.00     & 33.68$\pm$2.59     \\ \hline
\multirow{2}{*}{MNIST}                                        & NA                                        & 18.50$\pm$0.00     & 15.67$\pm$0.00    \\ \cline{2-4} 
                                                              & A                                         & 36.06$\pm$2.14     & 31.50$\pm$1.59     \\ \hline
\end{tabular}}
\end{table}}

{Table \ref{Tab:NetworkGeneral} shows the results of average equilibrium risks when the attacker attacks balanced networks, i.e., all the nodes in these networks have the same number of neighbors, with different numbers of nodes and degrees. Comparing the risks of networks with $6$ nodes but with different degrees, we can see that the attacker has more impact on networks with lower degrees as the risks are higher. Comparing the risks of the networks with $3$ nodes and $6$ nodes, we can see that the risks are higher when the attacker attacks networks with more nodes. Thus, a network with fewer nodes and a higher degree is more resilient.  In addition, the centralized SVMs under attacks have lower risks than DSVM, which indicates that DSVM is more vulnerable than centralized SVMs when the attacker compromises the whole training systems. }

{Table \ref{Tab:NetworkIndividual} shows the results of average equilibrium risks when the attacker attacks networks with 6 nodes and degree $0.33$. Note that one of the network is unbalanced, while the other network is balanced. In both networks, Node 1 is with the highest degree, while Node 6 is with the lowest degree. Comparing the results of attacking Node 1 and Node 6 in unbalanced network (or balanced network), we can see that the risks are higher when Node 1 is compromised. Thus, we can conclude that nodes with more neighbors tend to be more vulnerable.  Comparing the results of attacking Node 1 (or Node 6) in unbalanced network and balanced network, we can see that the risks are higher when the network is unbalanced no matter the attacker attacks higher degree nodes or lower degree nodes. Thus. we can conclude that balanced network tends to be more resilient to adversaries. 
\subsection{Weight of Node}
In the previous experiments, nodes in a network are considered to have the same number of training samples and testing samples. In this subsection, we study how the number of training samples affects the game between the learner and the attacker. We define the weight of a node as the number of training samples it contains. A higher weight means that the node contains more training samples.}

{From Table \ref{Tab:Network3Weight}, we can see that when Node 1 has more training samples, the risks become lower, which shows that the attacker has a smaller influence on the learner and the system is more secure. Though training more samples makes the system less vulnerable, it will require more time and more space for storage, which indicates that there is a trade-off between security and efficiency.  }
\section{Conclusion}
\label{Sec:Con}
Machine learning algorithms are ubiquitous but inherently vulnerable to adversaries. This paper has investigated the security issues of distributed support vector machines in an adversarial environment. We have established a game-theoretic framework to capture the strategic interactions between an attacker and a learner with a network of distributed nodes. We have shown that the nonzero-sum game is strategically equivalent to a zero-sum game, and hence its equilibrium can be characterized by a saddle-point equilibrium solution to a minimax problem. By using the technique of ADMoM, we have developed secure and resilient algorithms that can respond to the adversarial environment. We have shown that the convergence of the minimax problem to the equilibrium is guaranteed without the assumption of network topologies and the form of training data.

Experimental results have shown that an attacker can have a significant impact on DSVM if his capability and resources are sufficiently large. We have shown that the system itself can recover from attacks with the iterative and distributed nature of the algorithms. In addition, a network with a large number of nodes and a low degree is less secure. Hence, the network topology has a strong relation to the security of the DSVM algorithm. For a specified network, we have also shown that nodes with lower degrees are more secure. We have shown that a balanced network will be more secure, i.e., nodes in this network have similar degrees. We have also proved that adding more training samples will make the training process more secure.

{
One direction of future works is to develop a network design theory to form machine-learning networks that can achieve a desirable level of resiliency. In addition, we would also extend the current framework to investigate other machine learning algorithms, including nonlinear DSVM \cite{forero2010consensus}, large-scale SVMs \cite{chang2016semantic}, active learning \cite{yan2016image}, transfer learning\cite{zhang2018consensus}, and domain adaptions \cite{duan2012domain1}. We also intend to investigate other attack models, such as the cases when the attacker aims to increase the risk of a class of samples, he has limited knowledge, or he can modify training labels \cite{zhang2017game}, and so on. }

\section*{Appendix A: Proof of Proposition 1}
{
$\mathcal{U}_v$ is a sublinear aggregated action set of $\mathcal{U}_{v0}$ \cite{xu2009robustness}, and it satisfies ${\mathcal{U}^-} \subseteq \mathcal{U} \subseteq {\mathcal{U}^ + }$, where
\[\begin{array}{*{20}{c}}
{\begin{array}{*{20}{c}}
{{{\cal U}^ - } \buildrel \Delta \over = \mathop  \cup \limits_{t = 1}^n {\cal U}_t^ - ,{\rm{  }}{\cal U}_t^ -  \buildrel \Delta \over = \left\{ {\left( {{\delta _1},...,{\delta _n}} \right)\left|  \begin{array}{l}
{\delta _t} \in {{\cal U}_0};\\
{\delta _{i}} = {\bf{0}}, i \ne t.
\end{array} \right.} \right\}};\\ \medskip
{{{\cal U}^ + } \buildrel \Delta \over = \left\{ {\left( {{\alpha _1}{\delta _1},...,{\alpha _n}{\delta _n}} \right)\left| \begin{array}{l}
\sum\limits_{i = 1}^n {{\alpha _i} = 1} ;{\alpha _i} \ge 0,\\
{\delta _i} \in {{\cal U}_0},i = 1,...,n
\end{array} \right.} \right\}}.
\end{array}}
\end{array}\]
This property is used to prove Proposition $1$.  After reformulating Problem (\ref{eq:MinMax}) with hinge loss function, we can see that Problem (\ref{eq:alminmax}) and Problem (\ref{eq:MinMax}) are minimax problems with the same variables. Thus, we only need to prove that we minimize the same maximization problem. As a result, we only need to show that the following problem  
\begin{equation}
\label{eq:SubMaxA}
\begin{array}{l}
\mathop {\max }\limits_{\{\delta_{vn}\}\in \mathcal{U}_v} S( \{ \delta_{vn} \} ) \buildrel \Delta \over =  {V_a}{C_l}\sum\limits_{n = 1}^{{N_v}} {{{\left[ {1 - {{\rm{y}}_{vn}}({\bf{w}}_v^T({{\bf{x}}_{vn}} - {\delta _{vn}}) + {b_v})} \right]}_ + }} \\
\begin{array}{*{20}{c}}
{}&{}
\end{array} \ \ \ \ \ \ \ \ \ \ \ \ \ - {C_a}\sum\limits_{n = 1}^{{N_v}} {\left\| {{\delta _{vn}}} \right\|_{{0}} }
\end{array}
\end{equation}
is equivalent to the following problem 
\begin{equation}
\label{eq:SubMaxB}
\begin{array}{l}
\mathop {\max }\limits_{{\delta _v} \in {\mathcal{U}_{v0}}} {V_a}{C_l}\sum\limits_{n = 1}^{{N_v}} {{{\left[ {1 - {{\rm{y}}_{vn}}({\bf{w}}_v^T{{\bf{x}}_{vn}} + {b_v})} \right]}_ + }} + {V_a}{C_l}{\bf{w}}_v^T{\delta _v} \\
\ \ \ \ \ \ \ \ \ \ \ \ \ \ \ \ \  - {C_a}\left\| {{\delta _v}} \right\|_{{0}}.
\end{array}
\end{equation}
The first term of the objective function in (\ref{eq:alminmax}) and (\ref{eq:MinMax}) is ignored as they are not related to the maximization problem. Since $\{\delta_{vn} \}$ is independent in the maximization part of (\ref{eq:alminmax}), and $\delta_v$ is independent in the maximization part of (\ref{eq:MinMax}), we can separate the maximization problem into $V_a$ sub-maximization problems, and solving the sub-problems is equivalent to solving the global maximization problem. Therefore, we only need to show the equivalence between the sub-problem. }

{We adopt the similar proof in \cite{xu2009robustness}, recall the properties of sublinear aggregated action set, $\mathcal{U}_v^-\subseteq\mathcal{U}_v\subseteq\mathcal{U}_v^+ $. Hence, fixing any $({\bf{w}}_v,b_v)\in \mathbb{R}^{p+1}$, we have the following inequalities:
\begin{equation}
\label{eq:MainInequal}
\mathop {\max }\limits_{\{ \delta_{vn}\} \in {\cal U}_v^ - } S(\{\delta_{vn}\})  \le \mathop {\max }\limits_{\{ \delta_{vn}\} \in {\cal U}_v } S(\{\delta_{vn}\})  \le \mathop {\max }\limits_{\{ \delta_{vn}\} \in {\cal U}_v^ +  } S(\{\delta_{vn}\}) 
\end{equation}}

{
To prove the equivalence, we show that (\ref{eq:SubMaxB}) is no larger than the leftmost term and no smaller than the rightmost term of (\ref{eq:MainInequal}).
We first show that 
\begin{equation}
\label{eq:ProofAIn1}
\begin{array}{l}
\mathop {\max }\limits_{{\delta _v} \in {\mathcal{U}_{v0}}} {V_a}{C_l}\sum\limits_{n = 1}^{{N_v}} {{{\left[ {1 - {{\rm{y}}_{vn}}({\bf{w}}_v^T{{\bf{x}}_{vn}} + {b_v})} \right]}_ + }}  + {V_a}{C_l}{\bf{w}}_v^T{\delta _v} \\ \ \ \ \ \ \ \ \ \ \ \ \ \ \ - {C_a}\left\| {{\delta _v}} \right\|_0\\
 \le 
\mathop {\max }\limits_{({\delta _{v1}},...,{\delta _{v{N_v}}}) \in \mathcal{U}_v^ - } {V_a}{C_l}\sum\limits_{n = 1}^{{N_v}} {{{\left[ {1 - {{\rm{y}}_{vn}}({\bf{w}}_v^T({{\bf{x}}_{vn}} - {\delta _{vn}}) + {b_v})} \right]}_ + }}\\ \ \ \ \ \ \ \ \ \ \ \ \ \  - {C_a}\sum\limits_{n = 1}^{{N_v}} {\left\| {{\delta _{vn}}} \right\|_0.} 
\end{array}
\end{equation}
As the samples $\{\mathbf{x}_{vn},y_{vn}\}_{v=1}^{N_v}$ are not separable, there exists $t_v\in [1:N_v]$ which satisfies that \begin{equation}
\label{eq:ProofAIn1N}
{{\rm{y}}_{{t_v}}}({{\bf{w}}_v}^T{{\bf{x}}_{{t_v}}} + {b_v}) < 0.
\end{equation}
Hence, recall the definition of sublinear aggregated action set, we have:
\[\begin{array}{l}
\mathop {\max }\limits_{({\delta _{v1}},...,{\delta _{v{N_v}}}) \in \mathcal{U}_{vt_v}^ - } {V_a}{C_l}\sum\limits_{n = 1}^{{N_v}} {{{\left[ {1 - {{\rm{y}}_{vn}}({\bf{w}}_v^T({{\bf{x}}_{vn}} - {\delta _{vn}}) + {b_v})} \right]}_ + }} \\
\begin{array}{*{20}{c}}
{}&{}
\end{array} - {C_a}\sum\limits_{n = 1}^{{N_v}} {\left\| {{\delta _{vn}}} \right\|_0}\\
= \mathop {\max }\limits_{{\delta _{v{t_v}}} \in {\mathcal{U}_{v0}}} {V_a}{C_l}\sum\limits_{n \ne {t_v}} {{{\left[ {1 - {{\rm{y}}_{vn}}({\bf{w}}_v^T{{\bf{x}}_{vn}} + {b_v})} \right]}_ + }} \\
\begin{array}{*{20}{c}}
{}&{}
\end{array} + {V_a}{C_l}{\left[ {1 - {{\rm{y}}_{v{t_v}}}({\bf{w}}_v^T({{\bf{x}}_{v{t_v}}} - {\delta _{v{t_v}}}) + {b_v})} \right]_ + }\\
\begin{array}{*{20}{c}}
{}&{}
\end{array} - {C_a}\left\| {{\delta _{v{t_v}}}} \right\|_0\\
 = \mathop {\max }\limits_{{\delta _{v{t_v}}} \in {\mathcal{U}_{v0}}} {V_a}{C_l}\sum\limits_{n \ne {t_v}} {{{\left[ {1 - {{\rm{y}}_{vn}}({\bf{w}}_v^T{{\bf{x}}_{vn}} + {b_v})} \right]}_ + }} \\
\begin{array}{*{20}{c}}
{}&{}
\end{array} + {V_a}{C_l}{\left[ {1 - {{\rm{y}}_{v{t_v}}}({\bf{w}}_v^T{{\bf{x}}_{v{t_v}}} + {b_v})} \right]_ + }\\
\begin{array}{*{20}{c}}
{}&{}
\end{array} + {V_a}{C_l}({{\rm{y}}_{v{t_v}}}{\bf{w}}_v^T{\delta _{v{t_v}}}) - {C_a}\left\| {{\delta _{v{t_v}}}} \right\|_0\\
 = \mathop {\max }\limits_{{\delta _v} \in {\mathcal{U}_{v0}}} {V_a}{C_l}\sum\limits_{n = 1}^{{N_v}} {{{\left[ {1 - {{\rm{y}}_{vn}}({\bf{w}}_v^T{{\bf{x}}_{vn}} + {b_v})} \right]}_ + }} \\
\begin{array}{*{20}{c}}
{}&{}
\end{array} + {V_a}{C_l}{\bf{w}}_v^T{\delta _v} - {C_a}\left\| {{\delta _v}} \right\|_0
\end{array}\]
The second and third equalities hold because of Inequality (\ref{eq:ProofAIn1N}) and ${\mathop {\max }\limits_{{\delta _{{t_v}}} \in {{\cal U}_{v0}}} {\rm{(}}{{\rm{y}}_{{vt_v}}}{{\bf{w}}_v}^T{\delta _{{t_v}}})}$ being non-negative (recall $\mathbf{0}\in \mathcal{U}_{v0}$). Besides, we use $\delta_v$ to replace $\delta_{vt_v}$. Since $\mathcal{U}_{vt_v}^- \subseteq \mathcal{U}_v^- $, Inequality (\ref{eq:ProofAIn1}) holds.}

{
In the following step, we prove that
\begin{equation}
\label{eq:ProofAIn2}
\begin{array}{l}
\mathop {\max }\limits_{({\delta _{v1}},...,{\delta _{v{N_v}}}) \in \mathcal{U}_v^ + } {V_a}{C_l}\sum\limits_{n = 1}^{{N_v}} {{{\left[ {1 - {{\rm{y}}_{vn}}({\bf{w}}_v^T({{\bf{x}}_{vn}} - {\delta _{vn}}) + {b_v})} \right]}_ + }} \\
\begin{array}{*{20}{c}}
{}&{}
\end{array} - {C_a}\sum\limits_{n = 1}^{{N_v}} {\left\| {{\delta _{vn}}} \right\|_0} \\
 \le \\
\mathop {\max }\limits_{{\delta _v} \in {\mathcal{U}_{v0}}} {V_a}{C_l}\sum\limits_{n = 1}^{{N_v}} {{{\left[ {1 - {{\rm{y}}_{vn}}({\bf{w}}_v^T{{\bf{x}}_{vn}} + {b_v})} \right]}_ + }} \\
\begin{array}{*{20}{c}}
{}&{}
\end{array} + {V_a}{C_l}{\bf{w}}_v^T{\delta _v} - {C_a}\left\| {{\delta _v}} \right\|_0.
\end{array}
\end{equation}
Recall the definition of $\mathcal{U}^+$, we have:
\[\begin{array}{l} 
\mathop {\max }\limits_{({\delta _{v1}},...,{\delta _{v{N_v}}}) \in \mathcal{U}_v^ + } {V_a}{C_l}\sum\limits_{n = 1}^{{N_v}} {{{\left[ {1 - {{\rm{y}}_{vn}}({\bf{w}}_v^T({{\bf{x}}_{vn}} - {\delta _{vn}}) + {b_v})} \right]}_ + }} \\
\begin{array}{*{20}{c}}
{}&{}
\end{array} - {C_a}\sum\limits_{n = 1}^{{N_v}} {\left\| {{\delta _{vn}}} \right\|_0}  \end{array} \]
\[\begin{array}{l}
= \mathop {\max }\limits_{\scriptstyle\sum\nolimits_{n = 1}^{{N_v}} {{\alpha _{vn}} = 1} ;\hfill\atop
\scriptstyle{\alpha _{vn}} \ge 0;{\widehat \delta _{vn}} \in {\mathcal{U}_{v0}}\hfill} {V_a}{C_l}\sum\limits_{n = 1}^{{N_v}} {{{\left[ {1 - {{\rm{y}}_{vn}}({\bf{w}}_v^T({{\bf{x}}_{vn}} - {\alpha _{vn}}{{\widehat \delta }_{vn}}) + {b_v})} \right]}_ + }} \\
\begin{array}{*{20}{c}}
{}&{}
\end{array} - {C_a}\sum\limits_{n = 1}^{{N_v}} {\left\| {{\alpha _{vn}}{\delta _{vn}}} \right\|_0} \\
\le \mathop {\max }\limits_{\scriptstyle\sum\nolimits_{n = 1}^{{N_v}} {{\alpha _{vn}} = 1} ;\hfill\atop
\scriptstyle{\alpha _{vn}} \ge 0;{\widehat \delta _{vn}} \in {\mathcal{U}_{v0}}\hfill} {V_a}{C_l}\sum\limits_{n = 1}^{{N_v}} {{{\left[ {1 - {{\rm{y}}_{vn}}({\bf{w}}_v^T{{\bf{x}}_{vn}} + {b_v})} \right]}_ + }} \\
\begin{array}{*{20}{c}}
{}&{}
\end{array} + \sum\limits_{n = 1}^{{N_v}} {{V_a}{C_l}{\alpha _{vn}}{\bf{w}}_v^T{{\widehat \delta }_{vn}}}  - {C_a}\sum\limits_{n = 1}^{{N_v}} {\left\| {{\alpha _{vn}}{\delta _{vn}}} \right\|_0}\\
= \mathop {\max }\limits_{\scriptstyle\sum\nolimits_{n = 1}^{{N_v}} {{\alpha _{vn}} = 1} ;\hfill\atop
\scriptstyle{\alpha _{vn}} \ge 0.\hfill} \mathop {\max }\limits_{{{\widehat \delta }_{vn}} \in {\mathcal{U}_{v0}}} {V_a}{C_l}\sum\limits_{n = 1}^{{N_v}} {{{\left[ {1 - {{\rm{y}}_{vn}}({\bf{w}}_v^T{{\bf{x}}_{vn}} + {b_v})} \right]}_ + }} \\
\begin{array}{*{20}{c}}
{}&{}
\end{array} + {\alpha _{vn}}\sum\limits_{n = 1}^{{N_v}} {\left( {{V_a}{C_l}{\bf{w}}_v^T{{\widehat \delta }_{vn}} - {C_a}\left\| {{\delta _{vn}}} \right\|_0} \right)}\\
 = \mathop {\max }\limits_{{\delta _v} \in {\mathcal{U}_{v0}}} {V_a}{C_l}\sum\limits_{n = 1}^{{N_v}} {{{\left[ {1 - {{\rm{y}}_{vn}}({\bf{w}}_v^T{{\bf{x}}_{vn}} + {b_v})} \right]}_ + }} \\
\begin{array}{*{20}{c}}
{}&{}
\end{array} + {V_a}{C_l}{\bf{w}}_v^T{\delta _v} - {C_a}\left\| {{\delta _v}} \right\|_0.
\end{array}\]
Thus, Inequality (\ref{eq:ProofAIn2}) holds. By combining the two steps, we can show the equivalence between (\ref{eq:SubMaxA}) and (\ref{eq:SubMaxB}). Hence, Proposition \ref{proposition1} holds.}

\section*{Appendix B: Proof of Proposition 2}
To prove the equivalence, we use Neumann's Minimax Theorem\cite{nlkaido1954neumann}. Notice that solutions to the first and second constraints are convex for $\{\mathbf{r}_v,\xi_v\}$. The third constraints are linear equality functions. The forth constraint describes the set $\mathcal{U}_{v0}$ for $\delta_v$, which is convex set based on its definition. Thus we only need to prove that $K'$ is quasi-concave on $\{ \delta_{v}\}$ and quasi-convex on $\{ \mathbf{r}_{v}, \xi_v\}$.

On the one hand, the first two parts of $K'$ are constants for $\{ \delta_{v} \}$; the third part of $K'$ is a linear function of $\{\delta_{v}\}$; {the forth part of $K'$ is deleting the $l_1$ norm of $\{ \delta_{v} \}$, which is concave.} So $K'$ is a concave function for $\{ \delta_{v} \}$. Thus $K'$ is quasi-concave for $\{ \delta_{v} \}$. On the other hand, the first part of $K'$ is convex for $\{{{\bf{r}}_v}\}$ and linear for $\{\xi_v\}$; the second part of $K'$ is a linear function for $\{\xi_v\}$ and it is constant for $\{{{\bf{r}}_v}\}$; the third part of $K'$ is linear for $\{{{\bf{r}}_v}\}$ and it is constant for $\{\xi_v\}$; the forth part is constant for both $\{{{\bf{r}}_v}\}$ and $\{\xi_v\}$, so $K'$ is a convex function on $\{{{\bf{r}}_v}\}$. Thus $K'$ is quasi-convex on $\{{{\bf{r}}_v},\xi_v\}$. As a result, the equivalence holds. 

{
Since $K'$ is concave for $\{ \delta_{v} \}$ and convex for $\{{{\bf{r}}_v}\}$, there exists an equilibrium of the minimax or max-min problem \cite{basar1999dynamic}. Note that $l_1$ norm is not strictly convex, and thus $K'$ is not strictly concave for $\{ \delta_{v} \}$, so the equilibrium is not necessarily unique. }
\section*{Appendix C: Proof of Lemma 1}
We start with proving Inequality (\ref{eq:ConInequ3}) and Inequality (\ref{eq:ConInequ2}), and then we prove Inequality (\ref{eq:ConInequ1}).
\subsection*{Proof of Inequality (\ref{eq:ConInequ3})}
From Assumption 1, we have:
\begin{equation}
\label{eq:ConInequ3L0}
\begin{array}{l}
{L_0}({\{ {{\bf{r}}_v}\} ^ * },{\{ {\xi _v}\} ^ * },{\{ {\omega _{vu}}\} ^ * },{\{ {\alpha _{vu,k}}\} ^ * })\\
 \le {L_0}(\{ {{\bf{r}}_v}\} ,\{ {\xi _v}\} ,\{ {\omega _{vu}}\} ,{\{ {\alpha _{vu,k}}\} ^ * }).
\end{array}
\end{equation}
Since $\mathbf{r}_v^*-\omega_{vu}^*=0$ and $\omega_{vu}^*-\mathbf{r}_u^*=0$, the left side of (\ref{eq:ConInequ3L0}) becomes $p^*$. Thus, Inequality (\ref{eq:ConInequ3}) holds after introducing
\begin{equation}
\label{eq:ConInequ3pt1}
\begin{array}{l}
{p^{(t+1)}} = \frac{1}{2}\sum\limits_{v \in \mathcal{V}} {{\bf{r}}_v^{(t + 1)T}{\Pi _{p + 1}}{\bf{r}}_v^{(t+1)}} 
 + {V_a}{C_l}\sum\limits_{v \in \mathcal{V}_a} {{\bf{r}}_v^{(t + 1)T}{\widehat{\bf{I}}_{p\times(p + 1)}^T}\delta _v^*} \\ \ \ \ \ \ \ \ \ 
 + V{C_l}\sum\limits_{v \in \mathcal{V}} {\sum\limits_{n = 1}^{{N_v}} {{{\left[ {1 - {y_{vn}}[{\bf{x}}_{vn}^T,1]{\bf{r}}_v^{(t+1)}} \right]}_ + }} } .
\end{array}
\end{equation}
\subsection*{Proof of Inequality (\ref{eq:ConInequ2})}
From (\ref{eq:MinADMMi1}), $\mathbf{r}_v^{(t+1)}$ minimizes ${L_\eta }$. The necessary and sufficient optimality condition is:
\begin{equation}
\label{eq:ConInequ2Condition}
\begin{array}{l}
0 \in {\partial _{{\bf{r}}_v^{(t+1)}}}{L_\rho }({\{ {{\bf{r}}_v}\} ^{(t+1)}},{\{ {\omega _{vu}}\} ^{(t)}},{\{ {\alpha _{vu,k}}\} ^{(t)}})\\
 = {\Pi _{p + 1}}{\bf{r}}_v^{(t+1)} + {V_a}{C_l}{\widehat{\bf{I}}_{p\times(p + 1)}^T}\delta _v^*\\
 + V{C_l}\sum\limits_{n = 1}^{{N_v}} {{\partial _{{\bf{r}}_v^{(t+1)}}}{{\left[ {1 - {y_{vn}}[{\bf{x}}_{vn}^T,1]{{\bf{r}}_v^{(t+1)}}} \right]}_ + }}  \\
 + \sum\limits_{u \in {{\cal B}_v}} {(\alpha _{vu,1}^{(t)} - \alpha _{uv2}^{(t)})} \\
 + \sum\limits_{u \in {{\cal B}_v}} {\eta ({\bf{r}}_v^{(t+1)} - \omega _{vu}^{(t)})}  - \sum\limits_{u \in {{\cal B}_v}} {\eta (\omega _{uv}^{(t)} - {\bf{r}}_v^{(t+1)})} .
\end{array}
\end{equation}
Since $\alpha _{vu,1}^{(t+1)} = \alpha _{vu,1}^{(t)} + \eta r_{vu,1}^{(t+1)}$ and $
\alpha _{vu,2}^{(t+1)} = \alpha _{uv2}^{(t)} + \eta r_{uv2}^{(t+1)}$, we can plug $\alpha _{vu,1}^{(t)} = \alpha _{vu,1}^{(t+1)} - \eta r_{vu,1}^{(t+1)}$ and $\alpha _{uv2}^{(t)} = \alpha _{uv2}^{(t+1)} - \eta r_{uv2}^{(t+1)}$ in (\ref{eq:ConInequ2Condition}), and after rearranging it, we can obtain that $\mathbf{r}_v$ minimizes 
\begin{equation}
\label{eq:ConInequ2rmin}
\begin{array}{l}
\frac{1}{2}\sum\limits_{v \in \mathcal{V}} {{\bf{r}}_v^T{\Pi _{p + 1}}{{\bf{r}}_v}}  + {V_a}{C_l}\sum\limits_{v \in \mathcal{V}_a} {{\bf{r}}_v^T{\widehat{\bf{I}}_{p\times(p + 1)}^T}\delta _v^*} \\
 + V{C_l}\sum\limits_{v \in \mathcal{V}} {\sum\limits_{n = 1}^{{N_v}} {{{\left[ {1 - {y_{vn}}[{\bf{x}}_{vn}^T,1]{{\bf{r}}_v}} \right]}_ + }} }  + \sum\limits_{v \in \mathcal{V}} {\sum\limits_{u \in {{\cal B}_v}} {\alpha _{vu,1}^T{{\bf{r}}_v}} } \\
 - \sum\limits_{v \in \mathcal{V}} {\sum\limits_{u \in {{\cal B}_v}} {\alpha _{vu,2}^T{{\bf{r}}_u}} }  + \eta \sum\limits_{v \in \mathcal{V}} {\sum\limits_{u \in {{\cal B}_v}} {{\bf{r}}_v^T(\omega _{vu}^{(t+1)} - \omega _{vu}^{(t)})} } \\
 + \eta \sum\limits_{v \in \mathcal{V}} {\sum\limits_{u \in {{\cal B}_v}} {{\bf{r}}_v^T(\omega _{uv}^{(t+1)} - \omega _{uv}^{(t)})} } .
\end{array}
\end{equation}
Using the similar method, we can obtain that $\omega_{vu}^{(t+1)}$ minimizes 
\begin{equation}
\label{eq:ConInequ2omegamin}
\sum\limits_{v \in \mathcal{V}} {\sum\limits_{u \in {{\cal B}_v}} {\omega _{vu}^T( - \alpha _{uv1}^{(t+1)} + \alpha _{uv2}^{(t+1)})} } 
\end{equation}
Thus, we can obtain:
\begin{equation}
\label{eq:ConInequ2rminInequ}
\begin{array}{l}
\frac{1}{2}\sum\limits_{v \in \mathcal{V}} {{\bf{r}}_v^{(t + 1)T}{\Pi _{p + 1}}{\bf{r}}_v^{(t+1)}} 
 + {V_a}{C_l}\sum\limits_{v \in \mathcal{V}_a} {{\bf{r}}_v^{(t + 1)T}{\widehat{\bf{I}}_{p\times(p + 1)}^T}\delta _v^*} \\
 + V{C_l}\sum\limits_{v \in \mathcal{V}} {\sum\limits_{n = 1}^{{N_v}} {{{\left[ {1 - {y_{vn}}[{\bf{x}}_{vn}^T,1]{\bf{r}}_v^{(t+1)}} \right]}_ + }} }  + \sum\limits_{v \in \mathcal{V}} {\sum\limits_{u \in {{\cal B}_v}} {\alpha _{vu,1}^T{\bf{r}}_v^{(t+1)}} } \\
 - \sum\limits_{v \in \mathcal{V}} {\sum\limits_{u \in {{\cal B}_v}} {\alpha _{vu,2}^T{\bf{r}}_u^{(t+1)}} }  + \eta \sum\limits_{v \in \mathcal{V}} {\sum\limits_{u \in {{\cal B}_v}} {{\bf{r}}_v^{(t + 1)T}(\omega _{vu}^{(t+1)} - \omega _{vu}^{(t)})} } \\
 + \eta \sum\limits_{v \in \mathcal{V}} {\sum\limits_{u \in {{\cal B}_v}} {{\bf{r}}_v^{(t + 1)T}(\omega _{uv}^{(t+1)} - \omega _{uv}^{(t)})} } \\
 \le \frac{1}{2}\sum\limits_{v \in \mathcal{V}} {{\bf{r}}_v^{ * T}{\Pi _{p + 1}}{\bf{r}}_v^ * } 
 + {V_a}{C_l}\sum\limits_{v \in \mathcal{V}_a} {{\bf{r}}_v^{ * T}{\widehat{\bf{I}}_{p\times(p + 1)}^T}\delta _v^*} \\
 + V{C_l}\sum\limits_{v \in \mathcal{V}} {\sum\limits_{n = 1}^{{N_v}} {{{\left[ {1 - {y_{vn}}[{\bf{x}}_{vn}^T,1]{\bf{r}}_v^ * } \right]}_ + }} }  + \sum\limits_{v \in \mathcal{V}} {\sum\limits_{u \in {{\cal B}_v}} {\alpha _{vu,1}^T{\bf{r}}_v^ * } } \\
 - \sum\limits_{v \in \mathcal{V}} {\sum\limits_{u \in {{\cal B}_v}} {\alpha _{vu,2}^T{\bf{r}}_u^ * } }  + \eta \sum\limits_{v \in \mathcal{V}} {\sum\limits_{u \in {{\cal B}_v}} {{\bf{r}}_v^{ * T}(\omega _{vu}^{(t+1)} - \omega _{vu}^{(t)})} } \\
 + \eta \sum\limits_{v \in \mathcal{V}} {\sum\limits_{u \in {{\cal B}_v}} {{\bf{r}}_v^{ * T}(\omega _{uv}^{(t+1)} - \omega _{uv}^{(t)})} } .
\end{array}
\end{equation}
and
\begin{equation}
\label{eq:ConInequ2omegaminInequ}
\begin{array}{l}
\sum\limits_{v \in \mathcal{V}} {\sum\limits_{u \in {{\cal B}_v}} {\omega _{vu}^{(t + 1)T}( - \alpha _{uv1}^{(t+1)} + \alpha _{uv2}^{(t+1)})} } \\
 \le \sum\limits_{v \in \mathcal{V}} {\sum\limits_{u \in {{\cal B}_v}} {\omega _{vu}^{ * T}( - \alpha _{uv1}^{(t+1)} + \alpha _{uv2}^{(t+1)})} } .
\end{array}
\end{equation}

With $\mathbf{r}_v^*-\omega_{vu}^*=0$ and $\omega_{vu}^*-\mathbf{r}_u^*=0$, by adding the two inequalities above, we obtain Inequality (\ref{eq:ConInequ2}).
\subsection*{Proof of Inequality (\ref{eq:ConInequ1})}
Adding Inequality (\ref{eq:ConInequ3}) and Inequality (\ref{eq:ConInequ2}), by rearranging, we have:
\begin{equation}
\label{eq:ConInequ1Add}
\begin{array}{l}
2\sum\limits_{v \in \mathcal{V}} {\sum\limits_{u \in {{\cal B}_v}} {\left( {{{(\alpha _{vu,1}^{(t+1)} - \alpha _{vu,1}^ * )}^T}r_{vu,1}^{(t+1)} + {{(\alpha _{vu,2}^{(t+1)} - \alpha _{vu,2}^ * )}^T}r_{vu,2}^{(t+1)}} \right)} } \\
 + 2\sum\limits_{v \in \mathcal{V}} {\sum\limits_{u \in {{\cal B}_v}} {\left( {\eta {{(\omega _{vu}^{(t+1)} - \omega _{vu}^{(t)})}^T}r_{vu,1}^{(t+1)}} \right)} } \\
 - 2\sum\limits_{v \in \mathcal{V}} {\sum\limits_{u \in {{\cal B}_v}} {\left( {\eta (\omega _{vu}^{(t+1)} - \omega _{vu}^{(t)})r_{vu,2}^{(t+1)}} \right)} } \\
 + 4\sum\limits_{v \in \mathcal{V}} {\sum\limits_{u \in {{\cal B}_v}} {\left( {\eta (\omega _{vu}^{(t+1)} - \omega _{vu}^{(t)})(\omega _{vu}^{(t+1)} - \omega _{vu}^ * )} \right)} }  \le 0 .
\end{array}
\end{equation}

Using $\alpha_{vu,1}^{(t+1)}=\alpha_{vu,1}^{(t)}+\eta r_{vu,1}^{(t+1)},\alpha_{vu,2}^{(t+1)}=\alpha_{vu,2}^{(t)}+\eta r_{vu,2}^{(t+1)}$, $\alpha_{vu,1}^{(t+1)}-\alpha_{vu,1}^{(t)}=(\alpha_{vu,1}^{(t+1)}-\alpha_{vu,1}^*)-(\alpha_{vu,1}^{(t)}-\alpha_{vu,1}^*)$, $\omega_{vu}^{(t+1)}-\omega_{vu}^*=(\omega_{vu}^{(t+1)}-\omega_{vu}^{(t)})+(\omega_{vu}^{(t)}-\omega_{vu}^*)$, $\omega _{vu}^{(t+1)} - \omega _{vu}^{(t)} = (\omega _{vu}^{(t+1)} - \omega _{vu}^ * ) - (\omega _{vu}^{(t)} - \omega _{vu}^ * )$, (\ref{eq:ConInequ1Add}) is equivalent to
\begin{equation}
\label{eq:ConInequ1AddTran}
\begin{array}{l}
\frac{1}{\eta }\sum\limits_{v \in \mathcal{V}} {\sum\limits_{u \in {{\cal B}_v}} {\left( {\left\| {\alpha _{vu,1}^{(t+1)} - \alpha _{vu,1}^ * } \right\|_2^2 - \left\| {\alpha _{vu,1}^{(t)} - \alpha _{vu,1}^ * } \right\|_2^2} \right)} } \\
 + \frac{1}{\eta }\sum\limits_{v \in \mathcal{V}} {\sum\limits_{u \in {{\cal B}_v}} {\left( {\left\| {\alpha _{vu,2}^{(t+1)} - \alpha _{vu,2}^ * } \right\|_2^2 - \left\| {\alpha _{vu,2}^{(t)} - \alpha _{vu,2}^ * } \right\|_2^2} \right)} } \\
 + \eta \sum\limits_{v \in \mathcal{V}} {\sum\limits_{u \in {{\cal B}_v}} {\left\| {r_{vu,1}^{(t+1)} + (\omega _{vu}^{(t+1)} - \omega _{vu}^{(t)})} \right\|_2^2} } \\
 + \eta \sum\limits_{v \in \mathcal{V}} {\sum\limits_{u \in {{\cal B}_v}} {\left\| {r_{vu,2}^{(t+1)} - (\omega _{vu}^{(t+1)} - \omega _{vu}^{(t)})} \right\|_2^2} } \\
 + 2\eta \sum\limits_{v \in \mathcal{V}} {\sum\limits_{u \in {{\cal B}_v}} {\left\| {\omega _{vu}^{(t+1)} - \omega _{vu}^ * } \right\|_2^2} } \\
 - 2\eta \sum\limits_{v \in \mathcal{V}} {\sum\limits_{u \in {{\cal B}_v}} {\left\| {\omega _{vu}^{(t)} - \omega _{vu}^ * } \right\|_2^2} }  \le 0.
\end{array}
\end{equation}

Recall the definition of $J^{(t)}$, (\ref{eq:ConInequ1AddTran}) is equivalent to
\begin{equation}
\label{eq:ConInequ1AddTranJ}
\begin{array}{l}
{J^{(t+1)}} - {J^{(t)}} + \eta \sum\limits_{v \in \mathcal{V}} {\sum\limits_{u \in {{\cal B}_v}} {\left\| {r_{vu,1}^{(t+1)} + (\omega _{vu}^{(t+1)} - \omega _{vu}^{(t)})} \right\|_2^2} } \\
 + \eta \sum\limits_{v \in \mathcal{V}} {\sum\limits_{u \in {{\cal B}_v}} {\left\| {r_{vu,2}^{(t+1)} - (\omega _{vu}^{(t+1)} - \omega _{vu}^{(t)})} \right\|_2^2} }  \le 0.
\end{array}
\end{equation}

Since $\omega_{vu}^{(t+1)}$ minimizes $\sum\limits_{v\in\mathcal{V}} {\sum\limits_{u \in {{\cal B}_v}} {( - \alpha _{vu,1}^{(t+1)} + \alpha _{vu,2}^{(t+1)})\omega _{vu}^{}} } $, and  $\omega_{vu}^{(t)}$ minimizes $\sum\limits_{v\in\mathcal{V}} {\sum\limits_{u \in {{\cal B}_v}} {( - \alpha _{vu,1}^{(t)} + \alpha _{vu,2}^{(t)})\omega _{vu}^{}} } $, we have:
\begin{equation}
\label{eq:ConInequ1InequOmegat1}
\begin{array}{l}
\sum\limits_{v \in \mathcal{V}} {\sum\limits_{u \in {{\cal B}_v}} {{{( - \alpha _{vu,1}^{(t+1)} + \alpha _{vu,2}^{(t+1)})}^T}\omega _{vu}^{(t+1)}} } \\
 \le \sum\limits_{v \in \mathcal{V}} {\sum\limits_{u \in {{\cal B}_v}} {{{( - \alpha _{vu,1}^{(t+1)} + \alpha _{vu,2}^{(t+1)})}^T}\omega _{vu}^{(t)}} } .
\end{array}
\end{equation}
and
\begin{equation}
\label{eq:ConInequ1InequOmegat}
\begin{array}{l}
\sum\limits_{v \in \mathcal{V}} {\sum\limits_{u \in {{\cal B}_v}} {{{( - \alpha _{vu,1}^{(t)} + \alpha _{vu,2}^{(t)})}^T}\omega _{vu}^{(t)}} } \\
 \le \sum\limits_{v \in \mathcal{V}} {\sum\limits_{u \in {{\cal B}_v}} {{{( - \alpha _{vu,1}^{(t)} + \alpha _{vu,2}^{(t)})}^T}\omega _{vu}^{(t+1)}} } .
\end{array}
\end{equation}
Adding (\ref{eq:ConInequ1InequOmegat1}) and (\ref{eq:ConInequ1InequOmegat}) together, we arrive at 
\begin{equation}
\label{eq:ConInequ1InequOmegaADD}
\eta \sum\limits_{v \in \mathcal{V}} {\sum\limits_{u \in {{\cal B}_v}} {{{(r_{vu,1}^{(t+1)} - r_{vu,2}^{(t+1)})}^T}(\omega _{vu}^{(t+1)} - \omega _{vu}^{(t)})} }  \ge 0 .
\end{equation}
By unfolding the squares in (\ref{eq:ConInequ1AddTranJ}), we obtain Inequality (\ref{eq:ConInequ1}).

{
\bibliographystyle{ieeetr}
\bibliography{SecureDSVM.bib}}

\end{document}